\documentclass[9pt,technote]{IEEEtran}
\usepackage{cite}

\ifCLASSINFOpdf
   \usepackage[pdftex]{graphicx}
\else
\fi

\usepackage{epsfig}
\usepackage{epstopdf}
\usepackage{balance}
\usepackage{multicol}
\usepackage{widetext}
\usepackage{mathtools, cuted}
\usepackage{amsmath}
\usepackage{amsthm}
\usepackage{hhline}
\usepackage{multirow}
\usepackage{eucal}
\usepackage{cite}
\usepackage{ae}
\usepackage{amssymb}
\usepackage{subfigure}
\usepackage{color}
\newtheorem{thm}{Theorem}
\newtheorem{ass}{Assumption}
\newtheorem{lem}{Lemma}
\newtheorem{cor}{Corollary}

\newtheorem{property}{Property}

\newcommand*{\QEDA}{\hfill\ensuremath{\blacksquare}}%

\definecolor{airforceblue}{rgb}{0.36, 0.54, 0.66}
\definecolor{frenchblue}{rgb}{0.0, 0.45, 0.73}

\begin{document}
%
\title{Model-free Friction Observers for\\Flexible Joint Robots with Torque Measurements}
%
%
%

\author{Min~Jun~Kim, Fabian~Beck, Christian~Ott, and Alin~Albu-Sch\"affer
\thanks{The authors are with the Robotics and Mechatronics Center, German Aerospace Center (DLR),  Germany. Email: 
{\tt\small minjun.kim@dlr.de}}
}

\maketitle

\begin{abstract}
This paper tackles a friction compensation problem without using a friction model. The unique feature of the proposed friction observer is that the nominal motor-side signal is fed back into the controller instead of the measured signal. By doing so, asymptotic stability and passivity of the controller are maintained. Another advantage of the proposed observer is that it provides a clear understanding for the stiction compensation which is hard to be  captured in model-free approaches. This allows to design observers that do not overcompensate for the stiction. The proposed scheme is validated through simulations and experiments.
\end{abstract}


%
\IEEEpeerreviewmaketitle

\section{Introduction}
\label{section1}
%
%
%
%

Although a joint friction compensation is one of the most fundamental problems in robotics control, it is still an open problem; see e.g., \cite{verbert2016adaptive}. There are several branches in the friction observer studies. \cite{olsson1996observer, johanastrom2008revisiting,maged2019dynamic} proposed friction observers using the LuGre model that can describe most of physical phenomena of the friction. Model-based approaches will undoubtedly provide the best friction compensation performance, but at the cost of increased complexity. \cite{kaneko1990motion, ruderman2015observer} proposed model-free observers for motion control applications, but these may eliminate the interaction capability which is important in modern robotics.

To achieve friction compensation while preserving interaction capability, equipping the robot joints with a joint torque sensor (JTS) would be helpful because it allows to employ the flexible joint robot (FJR) model. In the FJR model, the overall robot dynamics are divided into the motor-side and link-side dynamics. Conceptually speaking, by applying a friction observer only to the motor-side dynamics, the link-side is still open to the interaction. Under the Spong's assumption \cite{spong1990modeling} which is valid when, for instance, the gear reduction ratio is high enough, the FJRs can be mathematically modeled as follows.
\begin{align}
\label{eq:link_dyn}
M(q)\ddot{q} + C(q,\dot{q})\dot{q} + g(q) &= \tau_j + \tau_{ext}, \\
\label{eq:motor_dyn}
B \ddot{\theta} + \tau_j &= \tau_m + \tau_f, \\
\label{eq:torque}
\tau_j &= K_j (\theta - q),
\end{align} 
where $M, C, g$ represent link-side inertia, Coriolis/centrifugal, gravity matrix/vector and $B$ represents motor-side inertia matrix. $q$, $\theta$ respectively denote the link-side and motor-side variables, and $K_j$, $\tau_j$ respectively denote the joint stiffness and joint torque which is measurable using JTS. $\tau_{ext}$ is the external torque resulting from interation, $\tau_m$ is the motor command, and $\tau_f$ is the joint friction to be compensated for.

An important remark here is that most of the significant friction can be included in the motor-side dynamics (\ref{eq:motor_dyn}) because the JTS is typically installed after the gear reduction. Therefore, it is reasonable to apply the friction observer to the motor-side dynamics which are governed by a 2nd order linear ordinary differential equation (w.r.t. $\theta$). Namely, we can apply linear control techniques to compensate for the friction. In this paper, friction observers are designed based on this observation.

\subsection{Related Work}
\label{sec:related_work}

The idea of applying friction compensation on the motor-side was realized in \cite{kaneko1990motion, zhang1997control,park2007disturbance} using disturbance observer (DOB) technique. In a very early study \cite{kaneko1990motion}, however, the joint torque information was not taken into account, meaning that the interaction on the link-side was treated as a disturbance. However, in robotics applications, it can be beneficial to close the loop around the motor-side dynamics using $\tau_j$  because then the link-side dynamics can interact with the environment through $\tau_{ext}$. \cite{zhang1997control} considered the joint torque information in the observer design, but the analysis was limited to a single-link robot. \cite{park2007disturbance} proposed an observer for multi-link robots, but the friction model was assumed to be linear and known.

To the best of the authors' knowledge, the approach proposed in \cite{le2008friction} was the first model-free friction observer for multi-link robotic systems (see Fig. \ref{fig:friction_observers}b). In this approach, the resulting observed value corresponds to the real friction smoothed by a 1st order low-pass filter (LPF). Despite successful experimental validation, however, theoretical analysis was not complete. The main challenge is the fact that the observer dynamics may break stability/passivity of the controller. Later, \cite{kim2014robust} and \cite{kim2015disturbance} proposed Fig. \ref{fig:friction_observers}c to establish a theoretically sound friction observer that guarantees stability of the whole system consisting of the friction observer dynamics and the FJR dynamics (\ref{eq:link_dyn})-(\ref{eq:torque}).\footnote{
To be precise, the scope of \cite{kim2014robust,kim2015disturbance} was about DOB-based control structures. The DOB becomes the friction observer when the motor inertia is known.
}

\begin{figure}[]
	\centering
	\subfigure[FJR control with the proposed friction observer]
	{\includegraphics[scale=0.28]{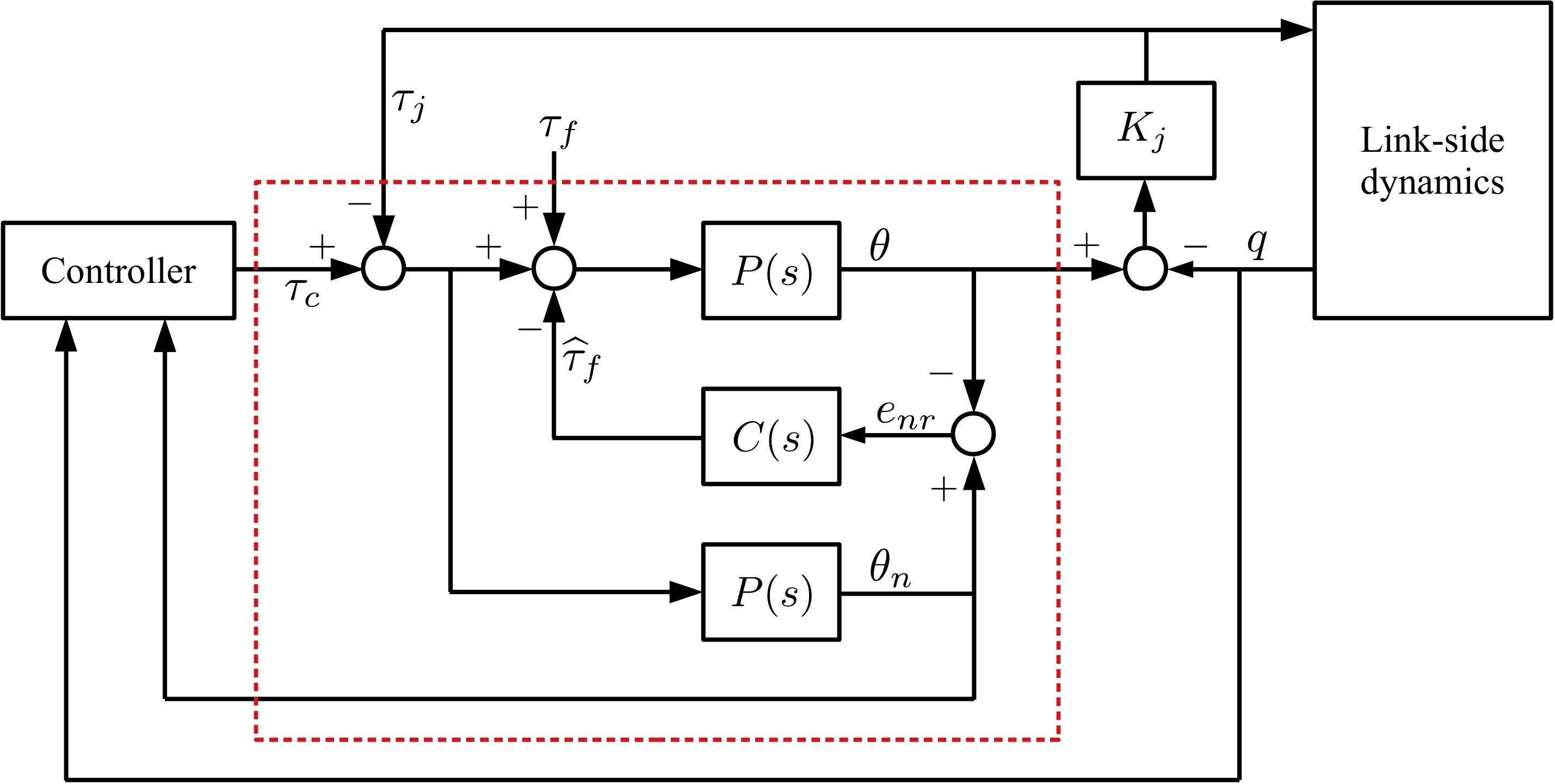}} \\
	\subfigure[Friction observer in \cite{le2008friction}]
	{\includegraphics[scale=0.24]{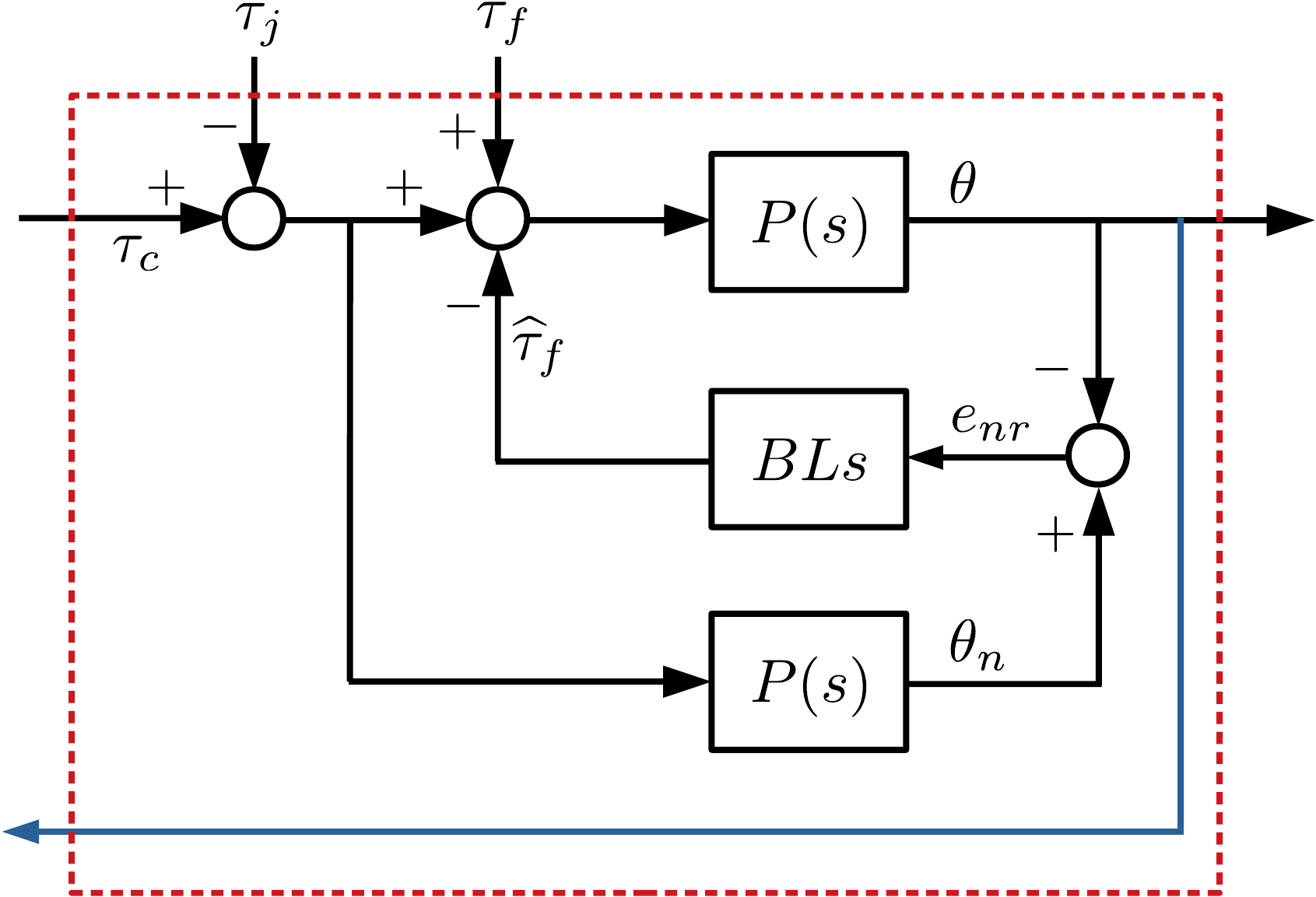}}
	\subfigure[Friction observer in \cite{kim2015disturbance}]
	{\includegraphics[scale=0.24]{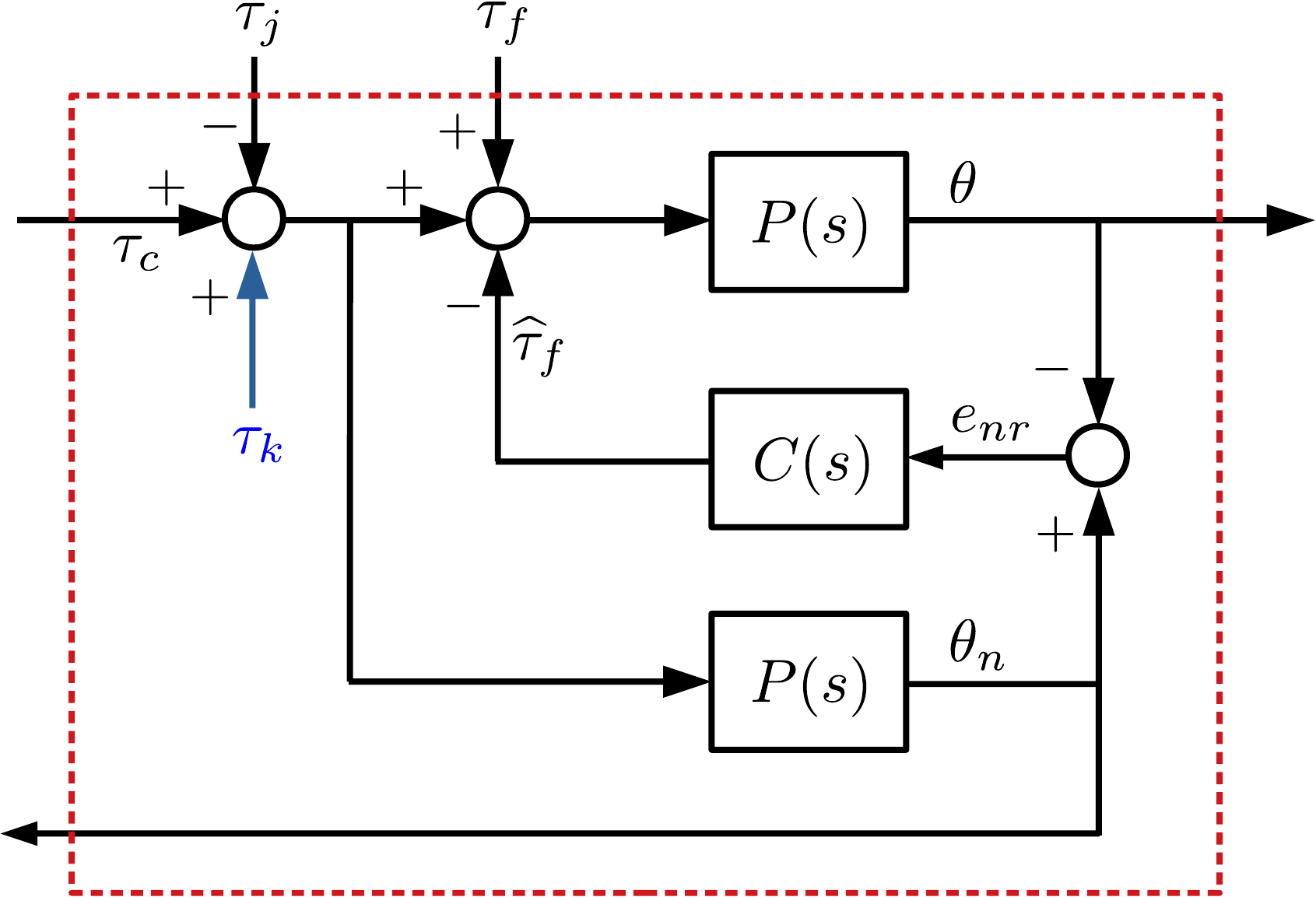}}
	\caption{ The dashed red boxes represent the friction observers, where $P(s)=1/Bs^2$ is the motor-side dynamics and $C(s)$ is the friction compensator. (a) The FJR control with the proposed friction observer.  (b) Friction observer proposed in \cite{le2008friction} where no stability proof is provided. (c) Friction observer proposed in \cite{kim2015disturbance}. Nominal motor signal ($\theta_n$) feedback makes the closed-loop system stable together with an auxiliary control input $\tau_k$. The limitation is that $\tau_k$ cannot be implemented around the equilibrium point. In (b) and (c), differences with the proposed observer (a) are highlighted with blue color.}
	\label{fig:friction_observers} 
\end{figure}

The main feature of this structure is that the nominal motor signal $\theta_n$ (instead of the measured signal $\theta$) is fed back into the controller. In Fig. \ref{fig:friction_observers}, $\theta_n$  can be thought of as the nominal motor signal because it is the outcome of the friction-free motor dynamics $P(s)$. Although asymptotic stability can be shown theoretically, the approach in Fig. \ref{fig:friction_observers}c still has a limitation that it contains an intractable input $\tau_k$ which cannot be implemented around the equilibrium point. Therefore, it was neglected in the experimental validation under the naive expectation that the influence is negligible \cite{kim2014robust}. Interestingly, the experiments were successful, and asymptotically stable behavior was shown. Nevertheless, it is questionable if $\tau_k$ can be neglected in general. This paper, therefore, shows that asymptotic stability can be guaranteed without $\tau_k$. Moreover, if the controller is designed to satisfy passivity of an input-output pair $(\tau_{ext},\dot{q})$, the friction observer preserves this property.

At this point, it is worth mentioning that the aforementioned approaches share a common keyword DOB. By virtue of the DOB-based structures, the outcomes of the observers are the low-pass filtered value of the real friction. However, the LPF property does not provide enough understanding for the observer's behavior, especially when the robot is stuck in the stiction. Stiction compensation needs a special treatment as it has a very complicated characteristics which is hard to be captured in the model-free approaches.

Indeed, one thing commonly missed in the model-free friction observer studies is the stiction compensation. This paper also neglects the dynamic behavior of the stiction in the beginning.  By doing so, it can be shown that the controlled system converges into the stiction region. However, the friction observers tend to generate energy during the stiction compensation, which may result in an oscillatory robot motion around the desired point. The proposed observer provides a better understanding on this phenomenon, and allows us to avoid the energy generation using passivity theory.

\begin{figure}
	\centering
	{\includegraphics[scale=0.45]{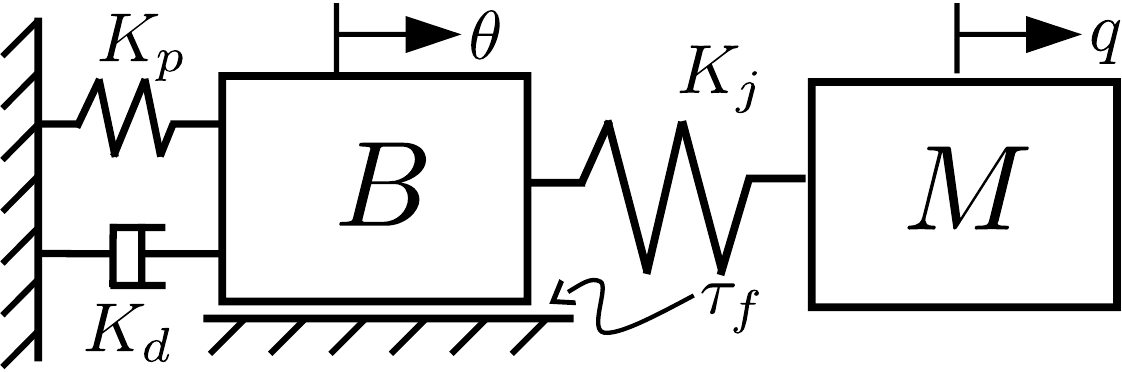}}
	\caption{A gravity-free single-link FJR with the motor-side PD controller. Friction observer is applied to compensate for the friction $\tau_f$ acting on the motor-side.}
	\label{fig:stiction_example} 
\end{figure}

\subsection{Scope and contribution of the paper}
\label{sec:scope_contribution}

This paper focuses mainly on the regulation scenario in which the robot eventually is stationary. Nevertheless, the tracking scenario is also discussed for the sake of completeness. We focus on the regulation case because the convergence can be argued only when the robot is supposed to be static. Convergence is hard to be concluded for the tracking case because the friction is observed by the relation of a LPF which always has a phase lag. Instead, we can claim practical stability which might be sufficient for tracking scenarios in practical point of view. 

Within the regulation scenario, this paper proposes friction observers that can be applied to controllers designed to be asymptotically stable for friction-free FJRs. Three main contributions are summarized as follows. First, it will be shown that asymptotic stability of the controller is preserved. Second, if, in addition, the controller is designed to satisfy passivity of $(\tau_{ext}, \dot{q})$, this is preserved as well. Third, a solution to prevent energy generation during the stiction compensation is presented. We would like to underline that, in the third contribution, the stiction region around the desired point is of interest in this paper.

\subsection{Organization of the paper}

The proposed friction observer is introduced in Section \ref{sec:friction_obs_design} with a theoretical justification in Section \ref{sec:theory}. Section \ref{sec:discussion} presents several features of the proposed friction observer as well as interpretations for the required technical conditions of the theory. The proposed features are validated in Section \ref{sec:validation} through simulations and experiments. Finally, Section \ref{sec:conclusion} concludes the paper.

\section{Model-Free Friction Observers for FJR}
\label{sec:friction_obs_design}

\subsection{Overview of the proposed friction observers}
\label{sec:proposed_fo}

This section introduces the proposed friction observers shown in Fig. \ref{fig:friction_observers}a without rigorous mathematical treatments which will be introduced in the subsequent sections.  Given a controller $\tau_c$ which is designed to be asymptotically stable for the nominal friction-free FJRs, the friction observer generates $-\widehat{\tau}_f$ to compensate for the friction $\tau_f$ in the motor-side dynamics (\ref{eq:motor_dyn}). Namely, the resulting motor command is
\begin{align}
\tau_m = \tau_c - \widehat{\tau}_f.
\label{eq:motor_command}
\end{align}

The following well-known motor-side PD controller \cite{tomei1991simple} is used as an example throughout the paper.
\begin{align}
	\tau_c(\theta, \dot{\theta}) = -K_p (\theta - \theta_d) - K_d \dot{\theta} + g(q_d),
	\label{eq:pd_controller}
\end{align}	
where $K_p$, $K_d$ are PD gains,  and $\theta_d = q_d + K_j^{-1}g(q_d)$ with the desired value $q_d$.  This controller is known to satisfy asymptotic stability and passivity of $(\tau_{ext},\dot{q})$ for sufficiently large $K_p$. When the friction observer is applied, as shown in Fig. \ref{fig:friction_observers}a, the controller should be defined using nominal signals:
\begin{align}
	\tau_c(\theta_n, \dot{\theta}_n) = -K_p (\theta_n - \theta_d) - K_d \dot{\theta}_n + g(q_d).
	\label{eq:pd_controller_with_nominal}
\end{align}	
In Fig. \ref{fig:friction_observers}a, $C(s)$ represents the friction observer. The following two observers are introduced in this paper.
\begin{align}
\label{eq:PID-type}
\text{PID-type:}& \;\; C(s) = - BL \left( s + L_p + \frac{L_i}{s} \right) \; \text{with} \; L_p^2> 2L_i, \\
\label{eq:PD-type}
\text{PD-type:}& \;\; C(s) = - BL \left( s + L_p \right),
\end{align}
where $L,L_p,L_i>0$ are observer gains.

It is important to note that, due to the additional observer dynamics, stability of the overall system may not be guaranteed even if the controller is designed to be stable for the nominal systems. This paper shows that the observer preserves asymptotic stability and/or passivity of the controller if the observer gain $L$ is sufficiently large for the regulation case; see Theorem \ref{thm:asymp_stable} in Section \ref{sec:stability} for the required technical conditions. When the tracking is of interest, friction-free behavior is achieved as the observer gain $L$ increases, but the convergence cannot be claimed for finite $L$; see Section \ref{sec:discuss_conditions} and Appendix.

One limitation of Theorem \ref{thm:asymp_stable} is the fact that dynamics of the stiction friction are neglected. Roughly speaking, the value of stiction friction has the same amount as net applied force/torque with the opposite sign. The following example motivates an additional analysis for the stiction compensation.

\underline{\it Motivating example:} Consider the gravity-free mass shown in Fig. \ref{fig:stiction_example} with the initial condition $\theta(0)=\dot{\theta}(0)=0$ and $q(0)=\dot{q}(0)=0$ (so that $\tau_j(0)=0$). The PD controller (\ref{eq:pd_controller}) or (\ref{eq:pd_controller_with_nominal}) is applied with $K_p=50$, $K_d=5$, and $\theta_d=0.01$. Let us assume that the maximum stiction value is $1.5 \mathrm{N}$. 
Because the PD control input (which is $50 \times 0.01 = 0.5 \mathrm{N}$ at the beginning) is not strong enough to break the stiction $1.5 \mathrm{N}$, the mass does not move due to the friction $-0.5 \mathrm{N}$. When the friction observer is applied, it produces $0.5 \mathrm{N}$ (i.e., $\widehat{\tau}_f=-0.5 \mathrm{N}$) to compensate for the friction. However, at the same time, the resulting friction increases to $-1 \mathrm{N}$ which is the net applied force. As a consequence, the friction compensation action will also increase to $1 \mathrm{N}$ which results in the net force $1.5 \mathrm{N}$ that can break the stiction. However, if the control error remains due to, for instance, stick-slip effect, the above mentioned procedure may be repeated. \QEDA

This example demonstrates that the friction observers tend to generate energy during the stiction compensation, and may result in an oscillatory motion. We would like to underline again that the stiction region around the desired regulating point is of main interest in this paper. An analysis for the stiction region that the robot enters instantly due to the change of velocity sign during the motion is out of scope.  

According to Theorem \ref{thm:passivity} in Section \ref{sec:behavior_in_stiction}, it is guaranteed that the PD-type observer does not generate the energy in the stiction region, whereas the PID-type may generate; see also Fig. \ref{fig:stiction_comp_stiction}a. However, we do not conclude which one is better in general. Section \ref{sec:validation} shows that PD- and PID-type observers have different characteristics with respective pros and cons. In practice, therefore, one should carefully design a friction observer depending on the application and hardware used.

\subsection{State-space representation of the overall dynamics}

To express the dynamics in state-space form, the most straightforward choice for the states would be $\theta_n$, $\dot{\theta}_n$, $\theta$, $\dot{\theta}$, $q$, and $\dot{q}$. In this paper, however, $e_{nr}=\theta_n - \theta$ and $\dot{e}_{nr}$  will be used instead of $\theta$ and $\dot{\theta}$, and will be collectively represented as $x_{nr}$.  The definition of $x_{nr}$ will depend on the friction observer $C(s)$ design. For example, when (\ref{eq:PID-type}) is used, one simple choice is $x_{nr}= [\int e_{nr}^T, \; e_{nr}^T, \; \dot{e}_{nr}^T]^T$. Similarly,  the state related to $\theta_n$ and $q$ will be collectively represented as $x_{n,q}$ of which the definition depends on the controller $\tau_c$. With the motor-side PD controller (\ref{eq:pd_controller_with_nominal}),  $x_{n,q}= [(\theta_n - \theta_d)^T, \; \dot{\theta}_n^T, \; (q - q_d)^T, \; \dot{q}^T]^T$ is a typical choice.

The overall dynamics can be expressed as 
\begin{align}
\label{eq:deriv_link}
M(q) \ddot{q} + C(q,\dot{q})\dot{q} + g(q) &= K_j(\theta-q) + \tau_{ext}, \\
\label{eq:deriv_nom_motor}
B\ddot{\theta}_n + K_j(\theta-q)& = \tau_c(x_{n,q}), \\
\label{eq:deriv_diff_dyn}
B\ddot{e}_{nr} &= \widehat{\tau}_f(x_{nr}) - \tau_f.
\end{align}
(\ref{eq:deriv_diff_dyn}) is called difference dynamics and can be obtained by subtracting (\ref{eq:motor_dyn}) from (\ref{eq:deriv_nom_motor}); recall also $\tau_m=\tau_c-\widehat{\tau}_f$ in (\ref{eq:motor_command}). Since $\theta$ is not a state, let us rewrite (\ref{eq:deriv_link})-(\ref{eq:deriv_nom_motor}) as
\begin{align}
\label{eq:nom_link_re}
M(q) \ddot{q} + C(q,\dot{q})\dot{q} + g(q) &= K_j(\theta_n-q) - K_j e_{nr} + \tau_{ext}, \\
\label{eq:nom_motor_re}
B\ddot{\theta}_n + K_j(\theta_n-q)& = \tau_c(x_{n,q}) + K_j e_{nr}.
\end{align}

In the state-space form, (\ref{eq:nom_link_re})-(\ref{eq:nom_motor_re}) can be represented as
\begin{align}
\label{eq:nq_dyn_ss}
\dot{x}_{n,q} = f_{n,q}(x_{n,q})+g_{n,q}(x_{nr}) + h_{n,q}(x_{nr})\tau_{ext}
\end{align}
with properly defined $f_{n,q}$, $g_{n,q}$, and $h_{n,q}$. Here, $g_{n,q}$ represents the perturbation caused by $K_je_{nr}$.  The state-space representation of (\ref{eq:deriv_diff_dyn}) is
\begin{align}
\dot{x}_{nr} = A_{nr} x_{nr} + B_{nr} \widehat{\tau}_f + B_{nr} w,
\label{eq:diff_ss}
\end{align}
with properly defined constant matrices $A_{nr}$ and $B_{nr}$.  $w$ represents the perturbation which may contain state dependent terms in addition to the friction $\tau_f$. By allowing $w$ to have state-dependent terms, the design procedure of the friction observer may become easier as we have more freedom.
\begin{lem}
	\label{lem:w_bound}
	Assume that $w$ is bounded by 
	\begin{align}
	||w|| \leq b_1 ||x_{nr}|| + b_2 ||\dot{\theta}|| + b_3 ||\dot{\theta}_n|| +  b_4,
	\label{eq:w_bound_condition}
	\end{align}
	for some $b_i \geq 0$. Then, $||w||^2$ can be bounded by
	\begin{align}
	||w||^2 \leq c_1 ||x_{nr}||^2 + c_2 ||x_{nr}|| +c_3 ||\dot{\theta}_n||^2 + c_4,
	\label{eq:w_sq_bound_condition}
	\end{align}
	for some $c_i \geq 0$.
\end{lem}
\begin{proof}
	First, use $\dot{\theta}=\dot{\theta}_n - \dot{e}_{nr}$ to get rid of $\dot{\theta}$ which is not a state of the closed-loop dynamics. Then, (\ref{eq:w_sq_bound_condition}) is obtained by algebraic calculation. When squaring $||w||$, a mixed term $||a||\cdot ||b||$ can be split using Young's inequality $ab \leq a^2/2 +b^2/2$.
\end{proof}

In the observer design, the following properties of the friction are required. LuGre model, for example, satisfies these \cite{johanastrom2008revisiting} . 
\begin{property} 
	\label{ass:friction}
	The friction $\tau_f(v)$ can be bounded by 
	\begin{align}
	||\tau_f|| \leq a_1 ||v|| + a_2,
	\label{eq:friction_bound}
	\end{align}
	for some $a_i>0$.
\end{property}
\begin{property}
	\label{prop:fric_passivity}
	The friction $\tau_f(v)$ defines a passive input-output pair $(v,-\tau_f(v))$.
\end{property}

\subsection{Friction observer designs}
\label{sec:stability}

This section presents the stability analysis for a class of controllers that satisfy the following assumption.
\begin{ass}
	\label{ass:V_nq}
	Let us consider the ideal friction-free system
	\begin{align}
	\label{eq:ideal_friction_free}
	\dot{x}_{n,q} = f_{n,q}(x_{n,q}) + h_{n,q}(x_{n,q}) \tau_{ext}.
	\end{align}
	When $\tau_{ext}=0$, there exists $V_{n,q}(x_{n,q}) \geq 0$ such that
	\begin{align}
	\label{eq:V_nq_partial}
	\left|\left| \frac{\partial V_{n,q}}{\partial x_{n,q}} \right|\right| \leq& \beta_{nq1}\phi_{nq}(||x_{n,q}||) + \beta_{nq2}, \\
	\label{eq:V_nq_dot}
	\dot{V}_{n,q} \leq& -\alpha_{nq} ||\dot{\theta}_n||^2,
	\end{align}
	where $\phi_{nq}(\cdot)$ is a positive definite function, and $\alpha_{nq}$, $\beta_{nq1}$, $\beta_{nq2}>0$ are some constants. Moreover, if $\tau_c$ is designed to satisfy the passivity of $(\tau_{ext},\dot{q})$, then there exists a storage function $V_{n,q} \geq 0$ that satisfies (\ref{eq:V_nq_partial}) and
	\begin{align}
	\label{eq:V_nq_dot_passive}	
	\dot{V}_{n,q} \leq -\alpha_{nq} ||\dot{\theta}_n||^2 + \tau_{ext}^T\dot{q} .
	\end{align}
\end{ass}
Note that $V_{n,q}$ does not necessarily have to be a Lyapunov function since positive definiteness is not required. This assumption is not restrictive as it can be satisfied for most of practical robot controllers, because $\phi_{nq}(\cdot)$ can be an arbitrary positive definite function and  $-\alpha_{nq} ||\dot{\theta}_n||^2$ comes from the D-control action which always exists to satisfy stability.\footnote{
For example, consider a typical choice of Lyapunov function defined by quadratic term plus gravity potential. Then (\ref{eq:V_nq_partial}) is trivial.} The controller (\ref{eq:pd_controller_with_nominal}) satisfies this assumption using (27) in \cite{tomei1991simple} as $V_{n,q}$ of which time derivative is given by
\begin{align}
\label{eq:V_dot_nq_pd_controller}
\dot{V}_{n,q}=-\dot{\theta}_n^T K_d \dot{\theta}_n + \tau_{ext}^T \dot{q}.
\end{align}

\begin{thm}
	\label{thm:asymp_stable}
	In addition to Assumption \ref{ass:V_nq}, assume that  (i) $\dot{\theta}_d=0$, (ii) the friction value is constant in stiction, and (iii) the controller $\tau_c$ is designed to be asymptotically stable for friction-free system (\ref{eq:ideal_friction_free}), and satisfies $\alpha_{nq} - \beta_{nr}c_3>0$. If the friction observer is designed to satisfy that
	\begin{enumerate}
		\item the difference dynamics (\ref{eq:diff_ss}) are exponentially stable when $w=0$, so that there exists $V_{nr}(x_{nr})>0$  (by Lyapunov converse Theorem \cite{khalil2002nonlinear}) satisfying
		\begin{align}
		\left|\left| \frac{\partial V_{nr}}{\partial x_{nr} } \right|\right| \leq& \bar{\beta}_{nr} ||x_{nr}||, \\
		\dot{V}_{nr} \leq& - \bar{\alpha}_{nr} ||x_{nr}||^2 ,
		\end{align}
		for some $\bar{\alpha}_{nr}, \bar{\beta}_{nr}>0$.  Therefore, when $w \neq 0$,
		\begin{align}
		\nonumber \dot{V}_{nr} \leq&  - \bar{\alpha}_{nr} ||x_{nr}||^2  + \bar{\beta}_{nr}||B_{nr}|| \cdot ||x_{nr}|| \cdot ||w|| \\
		\leq& - \alpha_{nr} ||x_{nr}||^2  + \beta_{nr}\cdot ||w||^2,
		\label{eq:V_nr_dot}
		\end{align}
		for some $\alpha_{nr},\beta_{nr}>0$,\footnote{$||x_{nr}|| \cdot ||w|| \leq \frac{1}{2}||x_{nr}||^2+\frac{1}{2}||w||^2$ using Young's inequality.}
		\item $w$ can be bounded by (\ref{eq:w_bound_condition}), and therefore (\ref{eq:w_sq_bound_condition}) is satisfied,
		\item $\alpha_{nr}$ is positive and increases with a certain observer gain, whereas $\beta_{nr}>0$ does not, 
		\item there exists a unique equilibrium point of $x_{nq}$ and $x_{nr}$,
	\end{enumerate}
	then the overall closed-loop dynamics (\ref{eq:deriv_nom_motor})-(\ref{eq:deriv_diff_dyn}) are asymptotically stable when the observer gain associated with $\alpha_{nr}$ is chosen sufficiently large with $\tau_{ext}=0$ . 
	
	Moreover, when $\tau_{ext} \neq 0$, if $V_{n,q}$ in Assumption \ref{ass:V_nq} is a storage function that satisfies (\ref{eq:V_nq_dot_passive}),  the passivity of $(\tau_{ext},\dot{q})$ is preserved for $||\tau_{ext}|| \leq b_{ext}<\infty$ for some $b_{ext}>0$.
\end{thm}
\begin{proof}
	See Section \ref{sec:proof_as}.
\end{proof}
Interpretations for the seven requirements in this theorem are addressed in Section \ref{sec:discuss_conditions}. Theorem \ref{thm:asymp_stable} concludes asymptotic stability of $x_{n,q}$ and $x_{nr}$, which is not the original goal of the controller. Asymptotic stability for $\theta$ can be concluded using $\theta=-e_{nr}+\theta_n$. The following corollaries propose a couple of friction observers.
\begin{cor}[PID-type $C(s)$]
	\label{cor:pid}
Consider $C(s)$ defined by (\ref{eq:PID-type}). If the observer gain $L$ is chosen sufficiently large, then the closed-loop dynamics of Theorem \ref{thm:asymp_stable} are asymptotically stable to  the equilibrium point $x_{n,q}=0$, $\dot{e}_{nr}=0$, $e_{nr}=0$, and $\int e_{nr}= L_i^{-1}L^{-1}B^{-1}\bar{\tau}_f$, where $\bar{\tau}_f$ is the friction at steady-state. 
\end{cor}
\begin{proof}
	See Section \ref{sec:proof_cor_pid}	
\end{proof}

\begin{cor}[PD-type $C(s)$]
	\label{cor:pd}
Consider $C(s)$ defined by (\ref{eq:PD-type}). If the observer gain $L$ is chosen sufficiently large, then the closed-loop dynamics of Theorem \ref{thm:asymp_stable} are asymptotically stable to $\dot{e}_{nr}=0$, $e_{nr}=L_p^{-1}L^{-1}B^{-1}\bar{\tau}_f$, but the equilibrium point of $x_{n,q}$ may be perturbed due to nonzero $e_{nr}$.
\end{cor}
\begin{proof}
	See Section \ref{sec:proof_cor_pd}
\end{proof}

It might be informative to mention that, for the proposed PID-/PD-type observers, $\alpha_{nr}$ increases and $\beta_{nr}$ decreases as the observer gain $L$ increases.

\subsection{Behavior of the friction observers in the stiction region}
\label{sec:behavior_in_stiction}

Theorem \ref{thm:asymp_stable} states asymptotic stability while neglecting dynamic behavior of the stiction. Therefore, it only implies that the robot trajectory converges to the stiction region around the desired point.  However, as illustrated in the motivating example in Section \ref{sec:proposed_fo},  dynamic behavior of the stiction may result in an oscillatory motion due to the energy generation during the stiction compensation (see Fig. \ref{fig:stiction_comp_stiction}a). To perform a further analysis, let us begin with the following assumption which is valid when the closed-loop dynamics have converged to the stiction region according to Theorem \ref{thm:asymp_stable}; hence $\tau_j=K_j(\theta-q)$ is constant because $\dot{\theta}=\dot{q} =0$.
\begin{ass}
	\label{ass:fric_passive}
	 Assume that, as a consequence of Theorem \ref{thm:asymp_stable}, the JTS measurement $\tau_j$ is constant. Then, the nominal motor-side dynamics (\ref{eq:deriv_nom_motor}) are not excited, and $\dot{\theta}_n=0$ can be further assumed. As a result, the friction can be written as $\tau_f(\dot{\theta}) = \tau_f(\dot{\theta}_n-\dot{e}_{nr}) = \tau_f(-\dot{e}_{nr})$.
\end{ass}

\begin{figure}
	\centering
	\begin{subfigure}[]
		{\includegraphics[height=2.61cm]{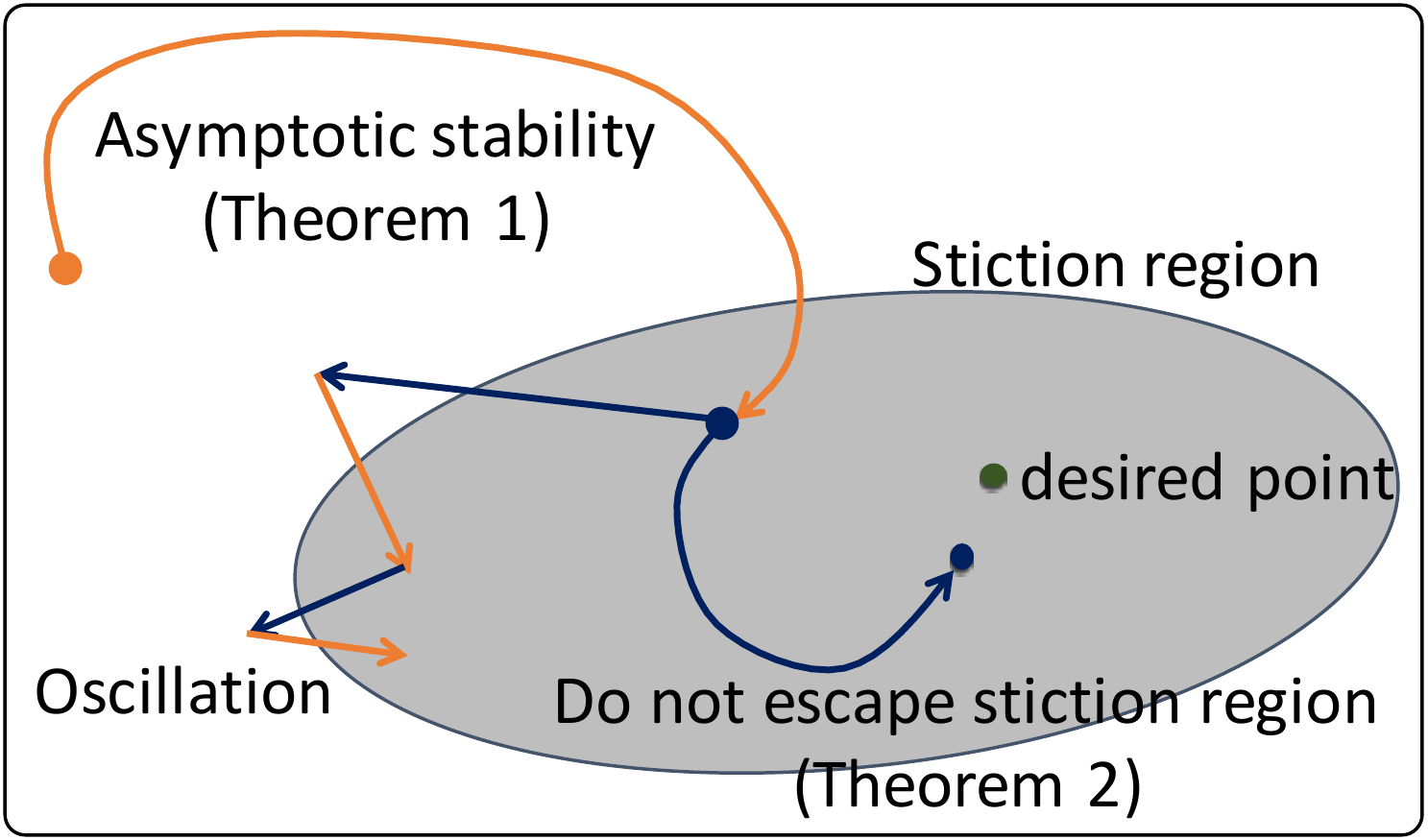}}
	\end{subfigure}
	\centering
	\begin{subfigure}[]
		{\includegraphics[height=2.74cm]{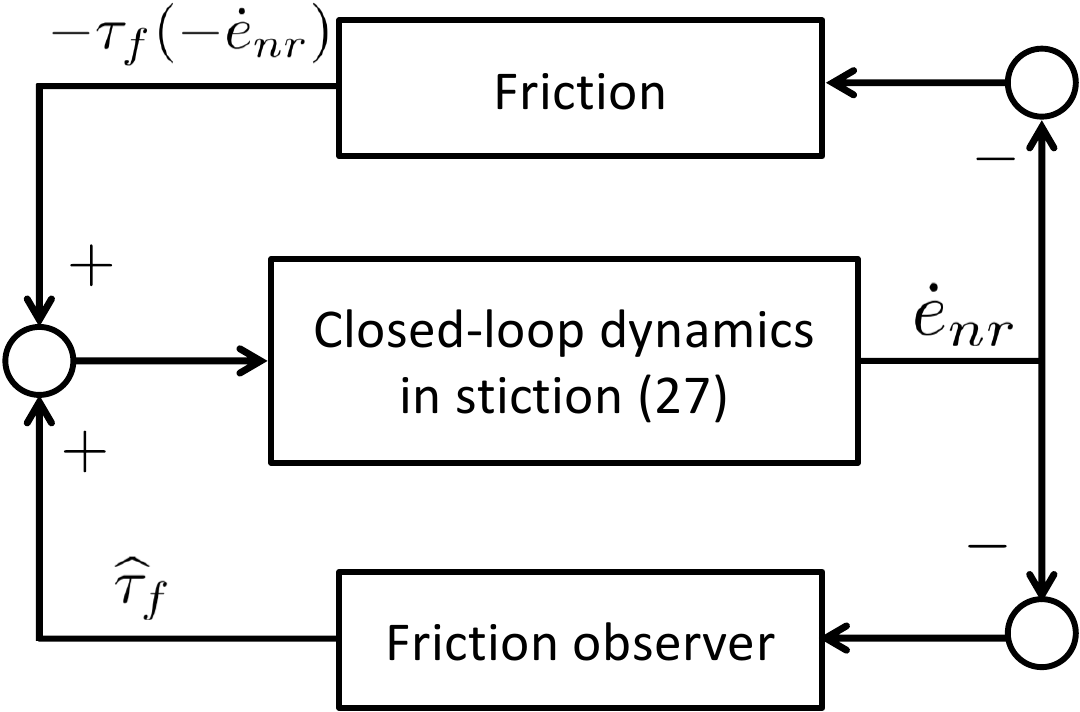}}
	\end{subfigure}
	\caption{(a) Conceptual roles of the main results. Although asymptotic stability (for regulation) is guaranteed by Theorem 1 while neglecting the dynamics of stiction, an oscillatory motion may occur due to the energy generation during stiction compensation. (b) Passivity-based analysis to prevent the energy generation. The resulting equilibrium point, however, may differ from the desired one.}
	\label{fig:stiction_comp_stiction} 
\end{figure}

It is now sufficient to investigate if the energy is generated in the closed-loop difference dynamics because the other sub-dynamics are not excited.
\begin{align}
\label{eq:deriv_diff_dyn_stiction}
B\ddot{e}_{nr} = \widehat{\tau}_f(x_{nr}) - \tau_f(-\dot{e}_{nr}).
\end{align}
Noting that the input-output pair $(-\dot{e}_{nr},-\tau_f(-\dot{e}_{nr}))$ of the friction dynamics is passive according to the Property \ref{prop:fric_passivity}, the following theorem is straightforward.
\begin{thm}
	\label{thm:passivity}
	Under Assumption \ref{ass:fric_passive}, if $\widehat{\tau}_f$ is designed to satisfy passivity of the input-output pair $(-\dot{e}_{nr}, \widehat{\tau}_f)$, then (\ref{eq:deriv_diff_dyn_stiction}) can be represented as feedback interconnections of passive subsystems.
\end{thm}
\begin{proof}
	Passivity of ``Closed-loop dynamics in stiction'' block is trivial because its dynamics are $1/Bs$, and that of ``Friction'' block is satisfied from the Property \ref{prop:fric_passivity}. If the ``friction observer'' block is also passive, the entire structure is represented by feedback interconnections of passive subsystems; see Fig. \ref{fig:stiction_comp_stiction}b.
\end{proof}
This theorem indicates that the energy is not generated during the stiction compensation if the friction observer is designed to be passive. Therefore, the PD-type observer (\ref{eq:PD-type}) guarantees that the closed-loop dynamics do not escape the stiction region, whereas the PID-type (\ref{eq:PID-type}) does not.

\section{Theoretical Derivation}
\label{sec:theory}

\subsection{Proof of Theorem \ref{thm:asymp_stable}}
\label{sec:proof_as}

Because the $\tau_c$ is a stabilizing controller for the nominal system (\ref{eq:ideal_friction_free}), the converse Lyapunov Theorem \cite{khalil2002nonlinear} guarantees the existence of $V_{n,q}^c>0$ (where the superscript $c$ stands for converse) such that
\begin{align}
\label{eq:V_qn_dot}
\dot{V}^c_{n,q} \leq -\alpha_{c} \phi_{1}(||x_{n,q}||) \;\; \text{and} \;\;
\left|\left|\frac{\partial V^c_{n,q}}{\partial x_{n,q}}\right|\right| \leq \beta_{c} \phi_{2}(||x_{n,q}||)
\end{align}
for some $\alpha_c>0$, $\beta_c>0$, and for some positive definite functions $\phi_{1}(\cdot)$ and $\phi_{2}(\cdot)$. Therefore, for the original system (\ref{eq:nq_dyn_ss}),  $\dot{V}_{n,q}+\dot{V}^c_{n,q}$ can be upper bounded as follows since $||g_{n,q}|| \leq \beta_g ||x_{nr}||$ for some $\beta_g>0$.
\begin{align}
	\dot{V}&_{n,q}+\dot{V}^c_{n,q} \leq - \alpha_{nq} ||\dot{\theta}_n||^2 -\alpha_{c} \phi_1(||x_{n,q}||) + \phi_3(||x_{n,q}||) ||x_{nr}||,
\label{eq:V_c_nq_dot}
\end{align}
where $\phi_3(\cdot) = \beta_g \left( \beta_c\phi_2(\cdot) + \beta_{nq1}\phi_{nq}(\cdot) + \beta_{nq2} \phi_{nq2} \right)$.

Let us define a Lyapunov-like function $V(x_{n,q}, x_{nr})=V_{n,q}(x_{n,q})+V^c_{n,q}(x_{n,q})+V_{nr}(x_{nr})>0$. Using (\ref{eq:V_nr_dot}) and(\ref{eq:V_c_nq_dot}),
\begin{align}
\nonumber \dot{V} \leq &  - \alpha_{nq} ||\dot{\theta}_n||^2 -\alpha_{c} \phi_1(||x_{n,q}||) + \phi_3(||x_{n,q}||) ||x_{nr}||  \\
& - \alpha_{nr} ||x_{nr}||^2 + \beta_{nr} ||w||^2.
\end{align}
 Using Lemma \ref{lem:w_bound},
\begin{align}
\label{eq:Vdot_first}
 \dot{V} \leq & -d_1 ||x_{nr}||^2  + d_2 ||x_{nr}|| + d_3 \\
\label{eq:Vdot_second}
&- \alpha_c \phi_1(||x_{n,q}||) - (\alpha_{nq} - \beta_{nr}c_3)||\dot{\theta}_n||^2 ,
\end{align}
where $d_1 = \alpha_{nr} -\beta_{nr}c_1$, $d_2 = \beta_{nr}c_2+ \phi_3(||x_{n,q}||)$, and $d_3 = \beta_{nr}c_4$. 
To save space, we borrow the result of Appendix from which the boundedness of the solution can be guaranteed. Based on this, convergence of $||x_{nr}||$, $||x_{n,q}||$, and $||\dot{\theta}_n||$ to certain values can be claimed because $d_2$ is bounded. We conclude the proof by showing that the states can only converge to the unique equilibrium point. First, $\dot{\theta}_n$ should converge to zero because otherwise $||x_{n,q}||$ cannot converge to a constant (hence contradiction). With $\dot{\theta}_n=0$, the difference dynamics become $B\ddot{e}_{nr} = \widehat{\tau}_f(x_{nr}) - \tau_f(-\dot{e}_{nr})$ because $\dot{\theta} = -\dot{e}_{nr} + \dot{\theta}_n=-\dot{e}_{nr} $. Hence $x_{nr}$ converges to the unique equilibrium, and this indicates the convergence of $x_{n,q}$ to its equilibrium as well.

Passivity  can be shown by simple extension. Nonzero $\tau_{ext}$ appears in time derivatives of $V_{n,q}$ and $V_{n,q}^c$. In $\dot{V}_{n,q}^c$, using (\ref{eq:V_qn_dot}), $\tau_{ext}$ appears as $\dot{V}_{n,q}^c \leq \beta_{c} \phi_{2}(||x_{n,q}||) ||\tau_{ext}||\leq b_{ext}\beta_{c} \phi_{2}(||x_{n,q}||)$ which eventually can be absorbed into  $d_3$ and plays no role by the same reasoning with the previous. Finally,  $\tau_{ext}^T\dot{q}$ included in $\dot{V}_{n,q}$ will result in  $\dot{V} \leq \tau_{ext}^T\dot{q}$, which indicates that the passivity is preserved.

\subsection{Proof of Corollary \ref{cor:pid}}
\label{sec:proof_cor_pid}

Let us express the closed-loop difference dynamics as
\begin{align}
\nonumber B (\ddot{e}_{nr} + L_p \dot{e}_{nr} + L_i e_{nr}) + \underbrace{R (\dot{e}_{nr} + L_p e_{nr} + L_i \int e_{nr})}_{=-\widehat{\tau}_f} \\
= \underbrace{-\tau_f + BL_p \dot{e}_{nr} + BL_i e_{nr}}_{=w},
\end{align}
where $R=BL$ is used for simplicity. Therefore, it is obvious that the first requirement of Theorem \ref{thm:asymp_stable} is true. The closed-loop difference dynamics can be expressed in the state-space with
\begin{align}
A_{nr}
=
\left[
\begin{array}{ccc}
0 & I & 0 \\
0 & 0 & I \\
0 &  -L_i & - L_p
\end{array}
\right], \;
B_{nr}
=
\left[
\begin{array}{c}
0 \\
0 \\
B^{-1}
\end{array}
\right].
\end{align}
Define $P$ and $Q$ as
\begin{align}
\label{eq:P_define}
P=
\left[
\begin{array}{ccc}
BL_i^2 + L_i L_p R & BL_p L_i + L_i R& BL_i \\
BL_p L_i  + L_i R   & BL_p^2 + L_p R & BL_p \\
BL_i   & BL_p & B
\end{array}
\right]
\end{align}
and $Q= diag\{ L_i^2 R, \;  (L_p^2-2L_i)R, \; R\}$. Then, $A_{nr}$, $B_{nr}$, $P$, and $Q$ satisfy the following:\footnote{
	Noting that (\ref{eq:riccati}) is the  Riccati equation associated with the $\mathcal{H}_\infty$ optimal control problem, the friction compensation input attenuates $w$ in the sense of $\mathcal{H}_\infty$ optimality \cite{kim2015bringing}:
	\begin{align}
\nonumber	\int x_{nr}^T Q x_{nr}^T + \widehat{\tau}_f^T R^{-1} \widehat{\tau}_f \leq \int w^T w,
	\end{align}
}
\begin{align}
A_{nr}^TP + P A_{nr} - P B_{nr} R B_{nr}^T P + Q = 0.
\label{eq:riccati}
\end{align}
Using this relation, we claim the following Lemma.

\begin{lem}
	\label{lem:nr}
	Using $V_{nr}=x_{nr}^T P x_{nr}>0$, we have
	\begin{align}
	\label{eq:bdd_V_nr_dot}
	\dot{V}_{nr} \leq -\lambda_{m}(Q)||x_{nr}||^2 + R^{-1} ||w||^2,
	\end{align}
	where $\lambda_{m}(Q)$ represents the minimum eigenvalue of $Q$.
\end{lem}
\begin{proof}
	$V_{nr}>0$ can be shown as follows:
	\begin{align}
	\nonumber V_{nr} =& (\dot{e}_{nr} + L_p e_{nr} + L_i\int e_{nr})^T B (\dot{e}_{nr} + L_p e_{nr} + L_i \int e_{nr}) \\
	+& 
	\left[
	\begin{array}{c}
	\int e_{nr} \\
	e_{nr} \\
	\end{array}
	\right]^T
	\underbrace{
		\left[
		\begin{array}{cc}
		L_i L_p R & L_i R \\
		L_i R & L_p R  \\
		\end{array}
		\right]
	}_{\triangleq P_1}
	\left[
	\begin{array}{c}
	\int e_{nr} \\
	e_{nr} \\
	\end{array}
	\right],		
	\end{align}
	where $P_1>0$ can be shown using Schur complement. Here, let us assume $\lambda_{m}(Q)=R$ for simplicity (the other cases can be analyzed similarly). Then,  
	\begin{align}
	\nonumber \dot{V}_{nr} =& 2x_{nr}^T P (A_{nr}x_{nr} + B_{nr} \widehat{\tau}_f + B_{nr}w) \\
	\nonumber =& 2x_{nr}^T P (A_{nr}x_{nr} - B_{nr} R B_{nr}^TPx_{nr} + B_{nr}w) \\
			 =& -x_{nr}^T Q x_{nr} - z^T R z + R^{-1} ||w||^2,
	\end{align}
	where $z=P B_{nr}  x_{nr} - R^{-1}w$. Hence (\ref{eq:bdd_V_nr_dot}) follows.
\end{proof}

Therefore, the second requirement of Theorem \ref{thm:asymp_stable} is true. Finally, the following Lemma tells that there is only one possible equilibrium point that the system state can converge to. Therefore, the third requirement in Theorem \ref{thm:asymp_stable} is also true.

\begin{lem}
	The equilibrium point of the overall dynamics (\ref{eq:deriv_nom_motor})-(\ref{eq:deriv_diff_dyn})  is uniquely given by $e_{nr}=\dot{e}_{nr}=0$, $\int e_{nr} = L_i^{-1} L^{-1}B^{-1} \bar{\tau}_f$ and $x_{n,q}=0$.
\end{lem}
\begin{proof}
	To begin with, note that $x_{n,q}=0$ is the equilibrium point of the ideal friction-free model (\ref{eq:ideal_friction_free}). Now, let us consider the steady state of the overall dynamics (\ref{eq:nq_dyn_ss})-(\ref{eq:diff_ss}). For (\ref{eq:diff_ss}), it is clear that the equilibrium point is $e_{nr}=\dot{e}_{nr}=0$ and $\int e_{nr} = L_i^{-1} L^{-1}B^{-1} \bar{\tau}_f$. Using $e_{nr}=0$, the equilibrium of (\ref{eq:nq_dyn_ss}) is that of (\ref{eq:ideal_friction_free}). Hence $x_{n,q}=0$ is the unique equilibrium point.
\end{proof}

\subsection{Proof of Corollary \ref{cor:pd}}
\label{sec:proof_cor_pd}

The following $P$ and $Q$ can be used.
\begin{align}
P=
\left[
\begin{array}{cc}
 BL_p^2 + B L_p R & B L_p \\
 B L_p & B
\end{array}
\right], \; \text{and} \; 
Q=
\left[
\begin{array}{ccc}
L_p^2 R & 0 \\
0 &  R
\end{array}
\right].
\end{align}
The same procedure as in Section \ref{sec:proof_cor_pid} can be followed.

\section{Discussions}
\label{sec:discussion}

\subsection{Working principle and generalization}
\label{sec:phyiscal_inter}

Combining (\ref{eq:PID-type}) with (\ref{eq:deriv_diff_dyn}) results in
\begin{align}
\label{eq:fric_obs_lpf_PID}
\widehat{\tau}_f = \frac{L s^2 + L L_p s + L L_i}{s^3 + L s^2 + L L_p s + L L_i}\tau_f,
\end{align}
which is the 3-2 order LPF.\footnote{$a$-$b$ order LPF means that it has $a$-th order denominator and $b$-th order numerator.}$^,$\footnote{Observer gains are treated as scalar in transfer functions. This simplification is possible when they are constant diagonal matrices.} Namely, the resulting $\widehat{\tau}_f$ is the low-pass filtered value of $\tau_f$. Similarly, the PD-type $C(s)$ in (\ref{eq:PD-type})  results in the 2-1 order LPF (use $L_i=0$ in (\ref{eq:fric_obs_lpf_PID})).

In addition to the LPF property, the proposed friction observer can be alternatively interpreted as follows: $C(s)$ is a control action that attenuates the influence of friction $\tau_f$ on $e_{nr}$ (or $x_{nr}$) in the difference dynamics (\ref{eq:deriv_diff_dyn}). The associated LPF can be derived from $C(s)$, and vice versa:
\begin{align}
\label{eq:relation_C_LPF}
LPF(s) = \frac{P(s)C(s)}{P(s)C(s)-1}, \;\; C(s) = \frac{LPF(s)}{P(s)\left\{LPF(s) -1 \right\} }.
\end{align}

\subsection{On the conditions of Theorem \ref{thm:asymp_stable}}
\label{sec:discuss_conditions}

This section presents interpretations for the seven conditions in Theorem \ref{thm:asymp_stable}. The conditions (i)-(iii) are assumptions  on the controller and friction, and conditions 1)-4) are requirements for the friction observer design.

{\it(i) $\dot{\theta}_d=0$:} As discussed in Section \ref{sec:phyiscal_inter}, the observer compensates for the friction by the relation of LPF which always has a phase lag. Namely, a perfect compensation cannot be achieved when the system is in motion, and therefore, we cannot say anything about convergence. This is a fundamental limitation of model-free approaches that cannot provide model-based friction feed-forward. To argue something about convergence, the system should be commanded to be stationary.

Practical stability can be shown for the tracking case, as presented in Appendix. We would like to underline that the tracking case does not require any conditions needed in Theorem \ref{thm:asymp_stable}. As far as the controller is stable (in any sense) and guarantees existence of a unique solution, practical stability can be concluded. It roughly means that the resulting trajectory approaches the unique solution as the observer gain increases.

{\it(ii) The friction value is constant in stiction:} Theorem \ref{thm:asymp_stable} states asymptotic stability while neglecting the dynamic behavior of the stiction friction. Namely, Theorem \ref{thm:asymp_stable} can only claim that the closed-loop dynamics are attracted to the stiction region, but there is no guarantee if the system stays there when the dynamics of stiction start to play a significant role. To answer this question, a passivity-based analysis on the stiction compensation is presented in Section \ref{sec:behavior_in_stiction}.

{\it(iii) The controller $\tau_c$ is nominally designed to be asymptotically stable, and satisfies $\alpha_{nq}-\beta_{nr}c_3>0$:} The link-side and nominal motor-side dynamics (\ref{eq:nom_link_re})-(\ref{eq:nom_motor_re}) are perturbed by $e_{nr}$. Therefore, the controller has to be strong enough to endure some amount of perturbation until the friction observer suppresses $e_{nr}$ sufficiently (recall also Section \ref{sec:phyiscal_inter}). A sufficient condition for this is $\alpha_{nq}-\beta_{nr}c_3>0$. When (\ref{eq:pd_controller_with_nominal}) is used with (\ref{eq:PID-type}), we have $\alpha_{nq}=K_d$ (recall (\ref{eq:V_dot_nq_pd_controller})), $\beta_{nr}=(BL)^{-1}$. $c_3$ is the coefficient for the viscous friction. In our experimental setup, $c_3 \simeq 10$, $B\simeq 1$. Hence the resulting sufficient condition was, roughly, $K_d \geq 10/L$. This was not restrictive in our experiments, because $L$ could be selected greater than 100.

{\it1) the difference dynamics (\ref{eq:diff_ss}) are exponentially stable when $w=0$, so that we have $\dot{V}_{nr} \leq -\alpha_{nr}||x_{nr}||^2 + \beta_{nr}||w||^2$:} Since the difference dynamics (\ref{eq:diff_ss}) are the 2nd order linear system, exponential stability is a reasonable requirement when $w=0$. In the proof, $-\alpha_{nr}||x_{nr}||^2$ is used to dominate $||w||^2$.

{\it2) $w$ can be bounded by (\ref{eq:w_bound_condition}), and therefore (\ref{eq:w_sq_bound_condition}) is satisfied:} The perturbation should be bounded not greater than linearly. Otherwise, the observer may not dominate the perturbation. We also remark that Assumption \ref{ass:V_nq} is introduced to dominate the perturbation $c_3 ||\dot{\theta}_n||^2$ in (\ref{eq:w_sq_bound_condition}), which is the only term that cannot be dominated by $-\alpha_{nr}||x_{nr}||^2$.

{\it3) $\alpha_{nr}$ is positive and increases with a certain observer gain, whereas $\beta_{nr}$ does not:} This requirement lies in the same line with the conditions 1) and 2). If $\alpha_{nr}$ increases with the observer gain while $\beta_{nr}$ does not, then the perturbation can be dominated by increasing the observer gain. Recall that for the proposed observers (\ref{eq:PID-type}) and (\ref{eq:PD-type}), $\alpha_{nr}=BL$ and $\beta_{nr}=(BL)^{-1}$.

{\it4) there exists a unique equilibrium point of $x_{nq}$ and $x_{nr}$:} Equilibrium point must be defined to claim stability.

\begin{figure*}
	\centering
	\begin{subfigure}[PID-type observer with low gains]
		{\includegraphics[width=5.85cm]{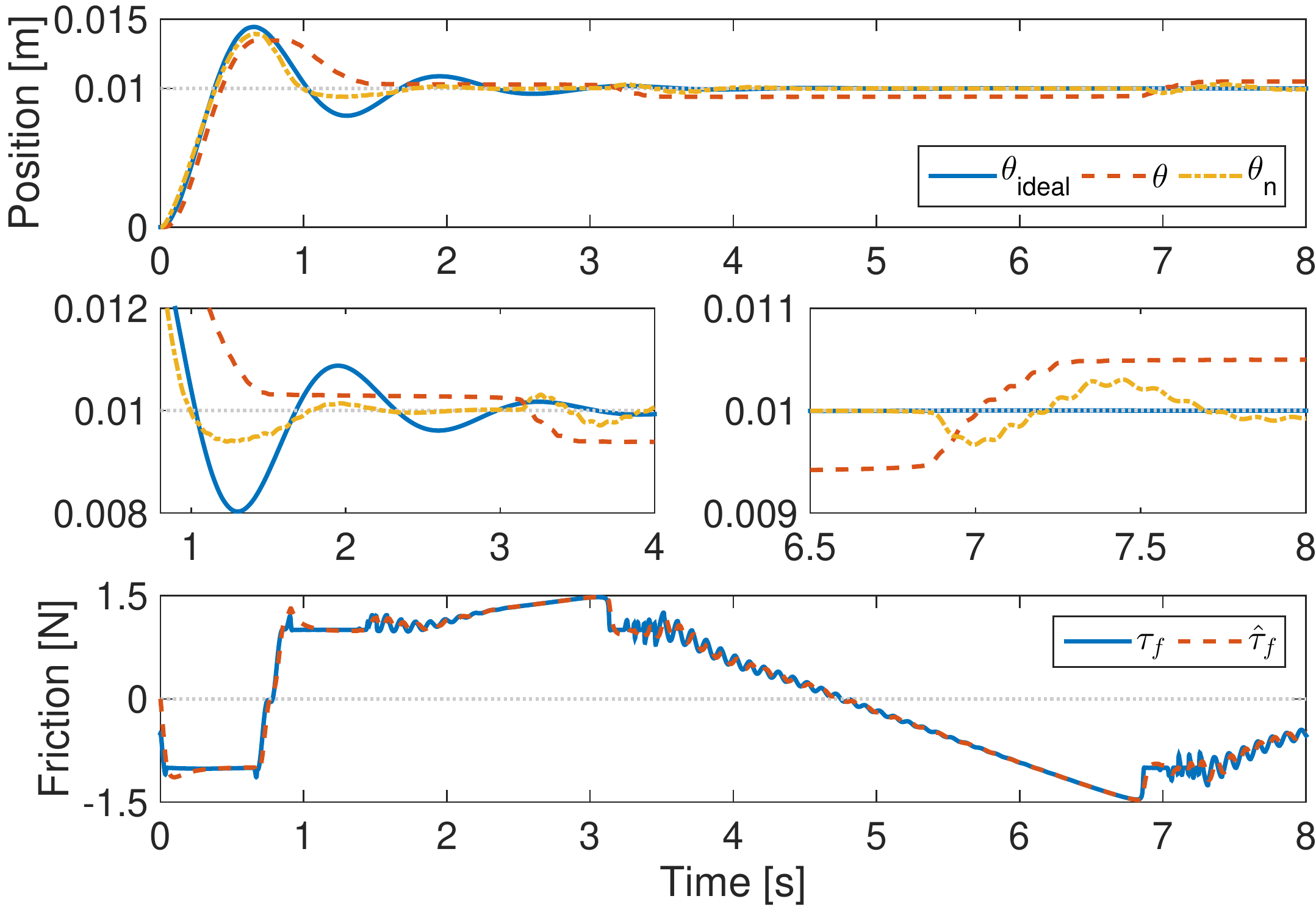}}
	\end{subfigure}
	\begin{subfigure}[PD-type observer with low gains]
		{\includegraphics[width=5.85cm]{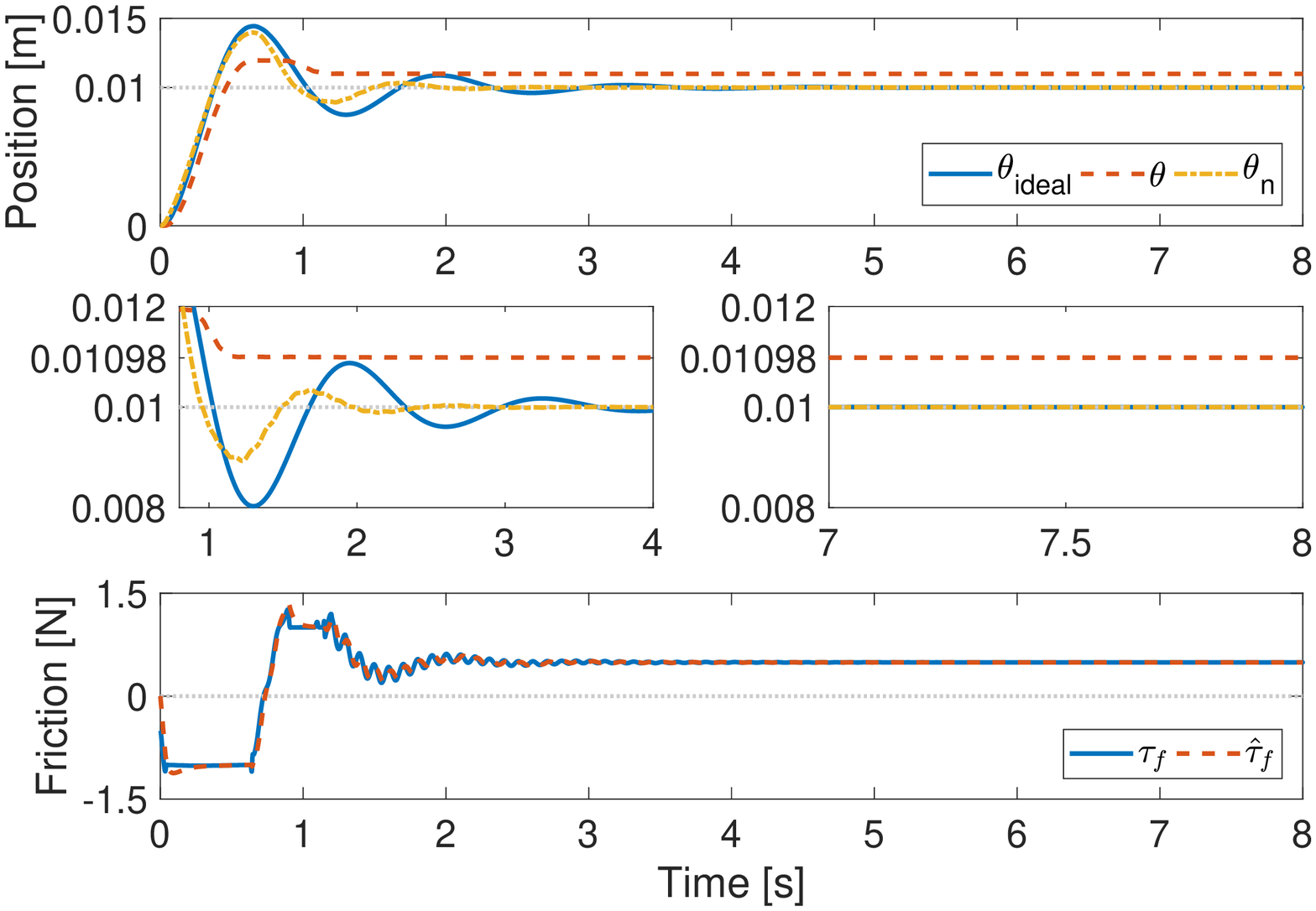}}
	\end{subfigure} 
	\begin{subfigure}[Observer from \cite{le2008friction} with low gains]
		{\includegraphics[width=5.85cm]{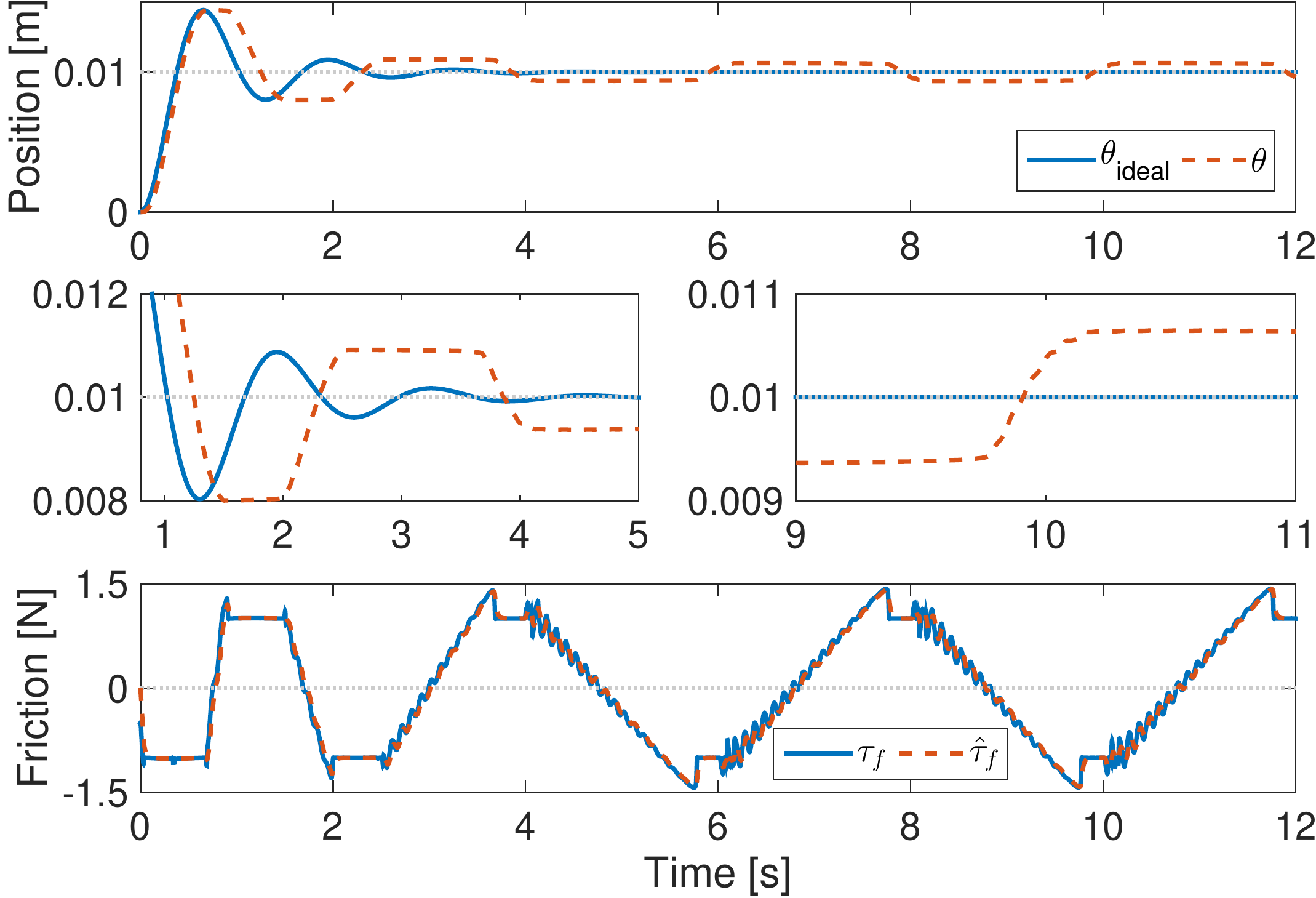}}
	\end{subfigure}\\
	\centering
	\begin{subfigure}[PID-type observer with high gains]
		{\includegraphics[width=5.85cm]{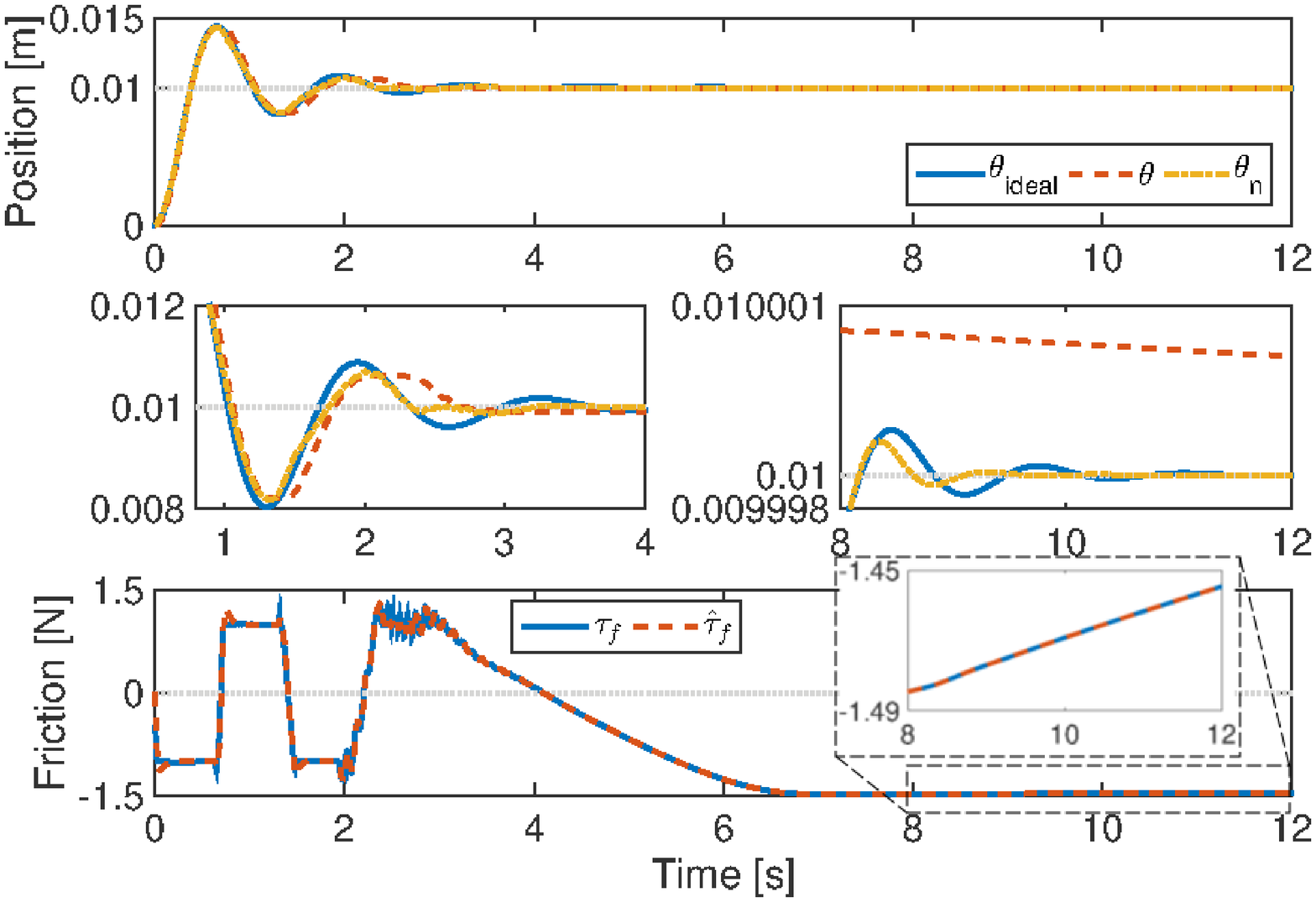}}
	\end{subfigure}	
	\begin{subfigure}[PD-type observer with high gains]
		{\includegraphics[width=5.85cm]{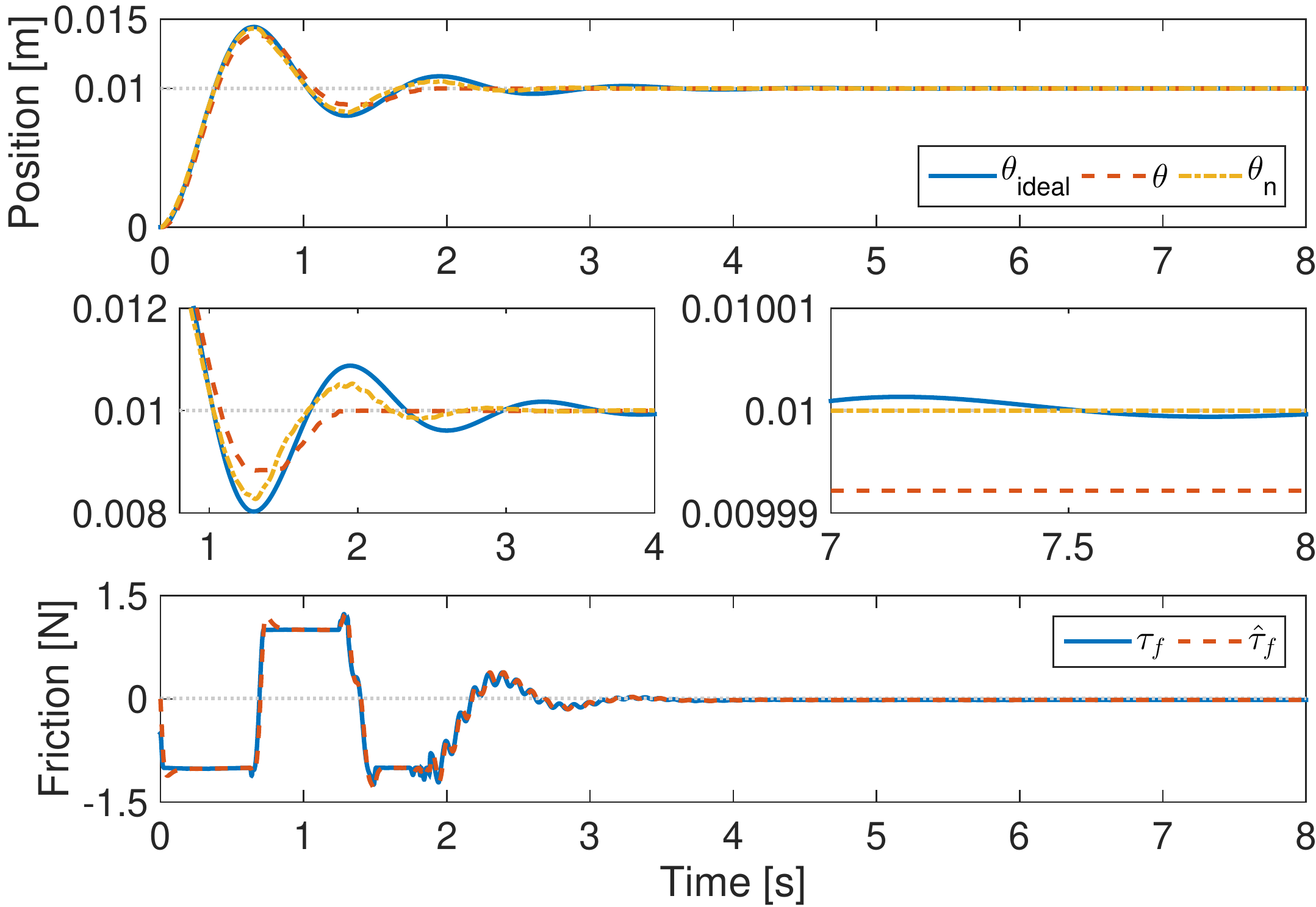}}
	\end{subfigure}
	\begin{subfigure}[Observer from \cite{le2008friction} with high gains]
		{\includegraphics[width=5.85cm]{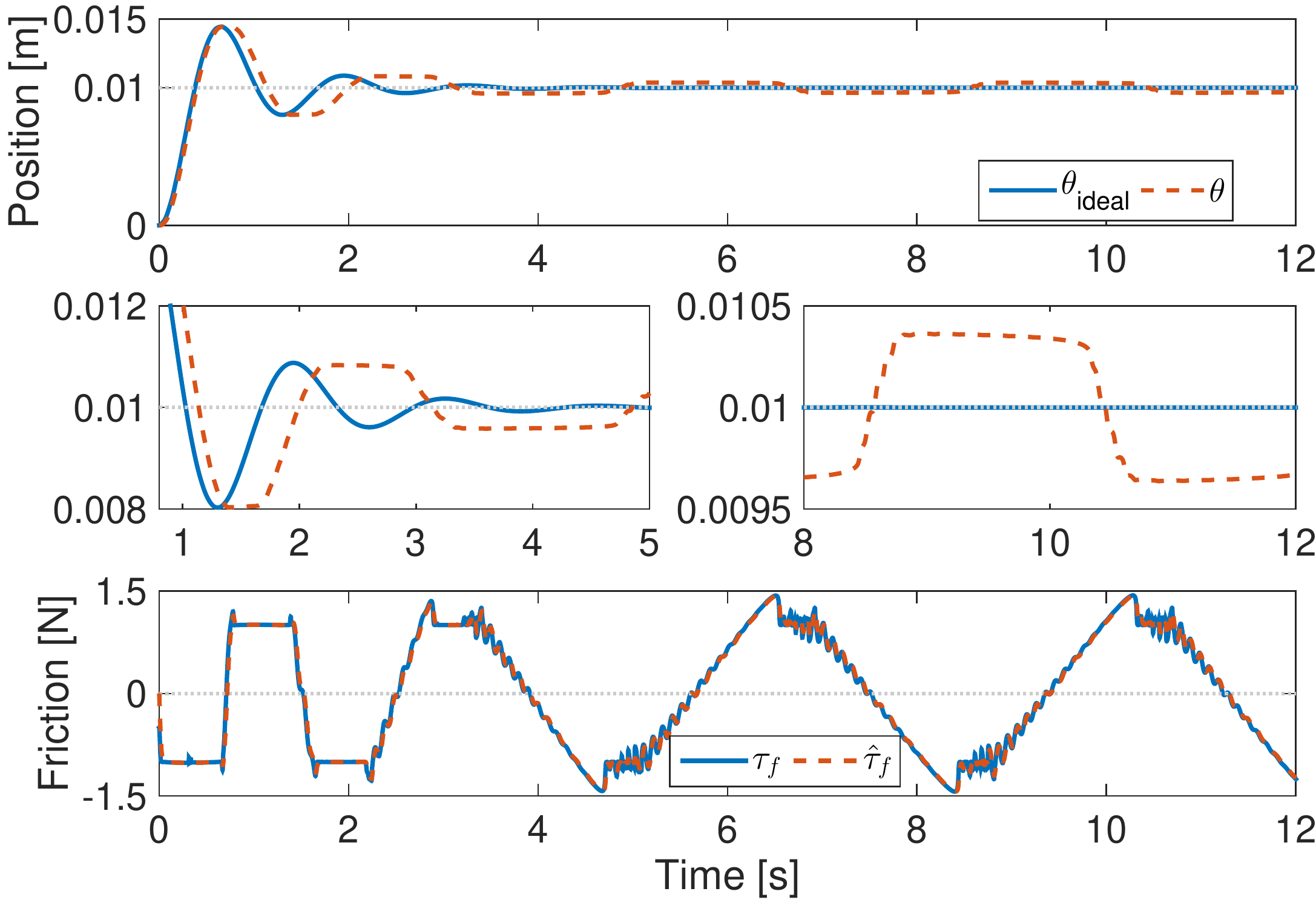}}
	\end{subfigure}
	\caption{Simulation results. For each plot, the top row shows the resulting motor-side positions with the magnified views in the second row. $\theta$ and $\theta_n$ are shown together with the ideal behavior $\theta_{\text{ideal}}$ which is simulated using friction-free ideal FJR. The third row shows $\tau_f$ and $\widehat{\tau}_f$.}
	\label{fig:simul} 
\end{figure*}

\subsection{Separation-like property}

In linear system theories, an observer is said to have separation property when the control gains and observer gains do not affect each other (hence can be designed separately). For the proposed friction observer, a similar property can be observed. In (\ref{eq:Vdot_second}), the decaying rate of $x_{n,q}$ is determined by  $\alpha_c$ which comes from the control design. Likewise, in (\ref{eq:Vdot_first}), the decaying rate of $||x_{nr}||$ is determined by $d_1$ in which the observer gain plays a dominant role when it is sufficiently large. Therefore the controller and friction observer can be designed separately.

\section{Validation}
\label{sec:validation}

\subsection{Simulations}

Simulations were performed to validate the proposed friction observer with the LuGre friction model:
\begin{align}
\label{eq:lugre_1}
\tau_f(v) = \sigma_0 z + \sigma_1 \dot{z} + \sigma_2 v \;\; {\text{with}} \;\;
 \dot{z} = v - \sigma_0  \frac{|v|}{g(v)}z, 
\end{align}
where $g(v) = f_c + (f_s - f_c)e^{-(v/v_s)^2}$. The parameters of the friction model are taken from \cite{de1995new}: $\sigma_0 = 10^5$, $\sigma_1 = \sqrt{10^5}$, $\sigma_2=0.4$, $f_c=1$, $F_s=1.5$, and $v_s=0.001$. The motivating example in Section \ref{sec:proposed_fo} was simulated with system parameters $B=1\mathrm{kg}$, $M=1\mathrm{kg}$, $K_j=3000\mathrm{N/m}$, $g(q)=0\mathrm{N}$. As illustrated in the example, PD gains were $K_p=50\mathrm{N/m}$ and $K_d=5\mathrm{N\cdot s/m}$. The desired position $\theta_d=0.01\mathrm{m}$ was commanded as step with zero initial conditions. Because the purpose of the simulation study is to validate the theoretical findings, we did not include any additional effects such as quantization, sensor noise, and link-side friction. The performance in the presence of these effects will be evaluated by the experiments.

Two sets of observer gains (low and high) were compared for both PID- and PD-type observers. The observer proposed in \cite{le2008friction} (shown in Fig. \ref{fig:friction_observers}b) was also implemented for comparison. $L=50$, $L_p=10$, $L_i=25$ were used for low gains, and $L=100$, $L_p=20$, $L_i=100$ were used for high gains. $L_i=0$ was set for the PD-type observer, and $L_p=L_i=0$ for the observer in \cite{le2008friction}. The simulation results are summarized in Fig. \ref{fig:simul}.

\begin{figure}
	\centering
	{\includegraphics[scale=0.39]{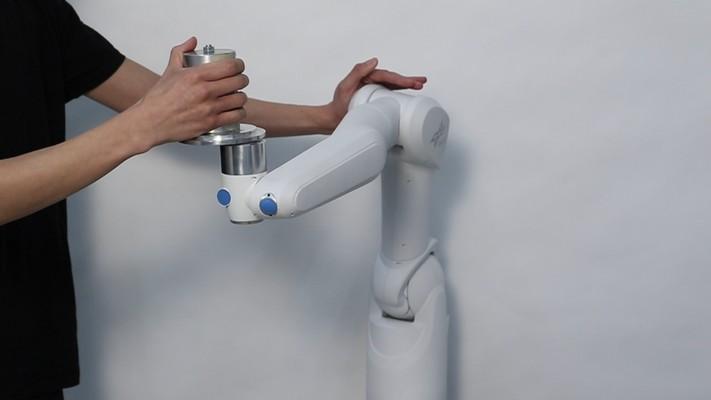}}
	\caption{7DOF surgical robot MIRO used in the experiments}
	\label{fig:miro} 
\end{figure}

When the low gains were used, as shown in Fig. \ref{fig:simul}a, the PID-type observer resulted in an oscillation around the desired equilibrium point. In the third row of the figure, $\widehat{\tau}_f$ varied together with $\tau_f$ due to the integrator since $e_{nr}$ was not zero. As illustrated in the motivating example, $\theta$ started to move when $\widehat{\tau}_f$ was large enough to break the stiction (see also the second row). In contrast, as shown in Fig. \ref{fig:simul}b, the oscillation did not occur for the PD-type observer which, however, resulted in constant steady-state error. Due to the absence of integral action, $\widehat{\tau}_f$ could converge to a certain value inside the stiction region, as shown in the third row. The observer in \cite{le2008friction} also showed oscillatory motion around the desired point (Fig. \ref{fig:simul}c)

When the PID-type was used with high gains, as shown in Fig. \ref{fig:simul}d, the oscillatory motion was hardly observed because the integral action did not play a significant role since $e_{nr}$ could be strongly suppressed. We report that, $e_{nr} \rightarrow 0$ (namely, constant $\widehat{\tau}_f$ and $\theta \rightarrow \theta_d$) could be achieved. For the PD-type, as shown in Fig. \ref{fig:simul}e, although the steady-state error was reduced due to the increased gains, it did not converge to zero, as expected from Corollary \ref{cor:pd}. For the observer in \cite{le2008friction}, the oscillatory motion remained although the amplitude was decreased.

From the simulation results, we can conclude that the PD- and the PID-type observers have different characteristics. The PD-type observer guarantees convergence to a certain point regardless of the observer gains, however, at the price of some steady-state error. PID-type observer may result in oscillatory motion, but there is also the possibility that it may outperform the PD-type observers.

\subsection{Experiments}

Two experiments were performed using a 7-link surgical robot MIRO shown in Fig. \ref{fig:miro}. The first experiment is designed to repeat the scenario of simulation studies to validate the theoretic findings. The second experiment shows the tracking scenario using task-space controller. Apart from these, the video attachment shows interaction of the robot with human operator as illustrated in Fig. \ref{fig:miro}. 

In the first experiment, the motor-side PD controller was used similar to the simulations with the following scenario: (i) initially, the robot was commanded to keep its initial position; (ii) at $t=3\mathrm{s}$, every joint was commanded to move $-5\mathrm{deg}$; (iii) at $t \approx 13\mathrm{s}$, a human operator applied external force. In addition to the proposed PID- and PD-type observer, the friction observer from \cite{le2008friction} was implemented for comparison. The case without friction observer is shown to provide reference. Similar to the simulation study, two sets of observer gains were used; low and high. Fig. \ref{fig:exp_joint} shows the error plot for the 2nd joint only, since this joint showed the different characteristics most clearly. For this joint, the control gains were $K_p=50\mathrm{Nm/rad}$ and $K_d=2\mathrm{Nm\cdot s/rad}$ with $B \approx 1 \mathrm{kg\cdot m^2}$; therefore, the joint should show under-damped motion if the friction is correctly compensated for. Similar to the simulation, with the PID-type observer, the oscillations occurred when the observer gains were low. When the high observer gains were used, the oscillations were barely observable. On the other hand, the PD-type observers did not show any oscillations for both low and high gains. Another interesting thing to observe is that the steady-state error was different after the human interaction because the joint can be stuck anywhere in the reduced stiction region. For the low gain case, for example, the steady-state error before the interaction was $0.003635\mathrm{rad}$ and after the interaction was $0.001881\mathrm{rad}$. When the friction observer from \cite{le2008friction} was used, asymptotically stable behavior was not achieved for both low and high gains.

In the second experiment, the following task-space impedance controller was implemented:
\begin{align}
\tau_c = J^{T}(\theta_n) \left( - K_p e_x - K_d \dot{e}_x\right)  + g(\bar{q}(\theta_n)),
\end{align}
where $J$ is the Jacobian matrix and $e_x=x(\theta_n) - x_d$. $x$ is a task-space variable which contains position/orientation of the end-effector as well as that of the third joint (in $\mathrm{rad}$). With the gravity compensation $g(\bar{q}(\theta_n))$ computed using quasi-statically estimated link-side position $\bar{q}$, this controller is known to satisfy asymptotic stability and passivity for regulation  \cite{ott2008passivity}. The control gains were selected as $K_p=1000\mathrm{N/m}$ , $K_d=200\mathrm{N \cdot s/m}$ for end-effector positions, $K_p=10\mathrm{Nm/rad}$, $K_d=2 \mathrm{Nm \cdot s/rad}$ for end-effector orientations and elbow. The end-effector was commanded to track a piecewise smooth (up to acceleration) square trajectory while keeping the end-effector orientation and elbow position fixed, in 5 seconds. The length of each line segment was $15\mathrm{cm}$. After $t=5\mathrm{s}$, the desired end-effector position was set to be constant. Compared to the case without friction observer, the friction observers contributed to improve the tracking performance, as shown in Fig. \ref{fig:exp_task}. 

The video attachment of this paper shows the second experiment with human interaction. In the video, the control gains $K_p$, $K_d$ of the end-effector position were lowered in order to show the effect of friction observers and to show interaction capability more clearly.

\begin{figure}
	\centering
	{\includegraphics[scale=0.315]{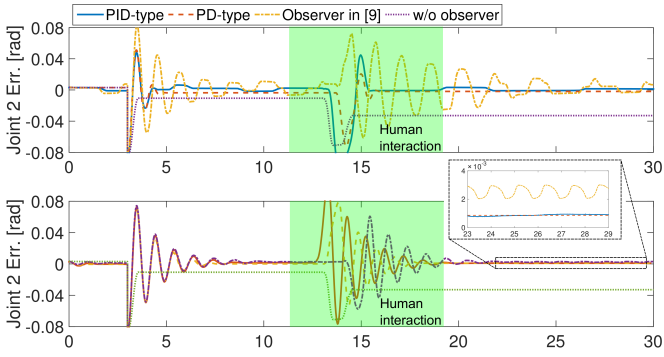}}
\caption{Experimental results for the joint space PD regulation control. Step command ($5^\circ$) was commanded to all joints at $t=3\mathrm{s}$, and a human operator applied external force at $t \simeq 13\mathrm{s}$. Top row: low observer gains, bottom row: high observer gains.}
	\label{fig:exp_joint} 
\end{figure}

\begin{figure}
	\centering
	{\includegraphics[scale=0.215]{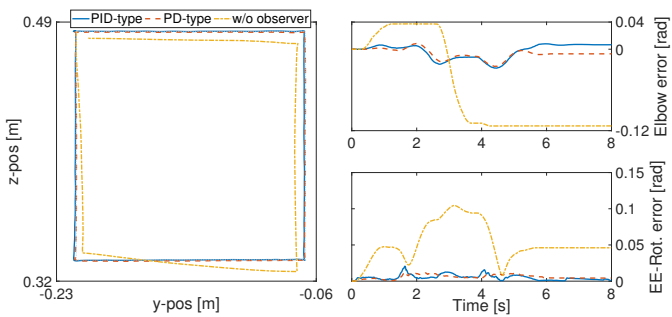}}
	\caption{Experimental results for the task space tracking control. }
	\label{fig:exp_task} 
\end{figure}

\section{Summary and Conclusion}
\label{sec:conclusion}

This paper proposed model-free friction observers that do not harm stability/passivity of the original controller. The main feature of the proposed approach is the use of nominal motor signals in the control law (Fig. \ref{fig:friction_observers}a). A similar friction observer was proposed in the previous study, but there was an intractable input ($\tau_k$ in Fig. \ref{fig:friction_observers}c) which cannot be implemented around the equilibrium point. For this reason, $\tau_k$ has been neglected under the expectation that it has little influence on stability. This paper theoretically clarifies that $\tau_k$ can be neglected in general. The proposed approach has another advantage that it provides a better understanding on the stiction compensation during which the friction observers tend to generate energy. In particular, it is possible to design a passive friction observer that does not generate  energy. The proposed scheme is validated through simulations and experiments.

\appendix
To show practical stability when the tracking is of interest, singular perturbation theory (in particular, Tikhonov's Theorem) can be applied \cite{khalil2002nonlinear}. Consider the PID-type $C(s)$ in Corollary \ref{cor:pid}. Letting $L=\frac{1}{\epsilon}I$, the boundary layer system can be obtained by taking time derivative of $\widehat{\tau}_f$:
\begin{align}
\epsilon\dot{\widehat{\tau}}_f + \widehat{\tau}_f = \tau_f + L_p \dot{e}_{nr} + L_i e_{nr}.
\label{eq:boundary_layer}
\end{align}
In (\ref{eq:boundary_layer}), as $\epsilon \rightarrow 0$, exponential convergence of $\widehat{\tau}_f=\tau_f + L_p \dot{e}_{nr} + L_i e_{nr}$ is achieved in the fast time scale $t/\epsilon$ (notice that $e_{nr},\dot{e}_{nr},\tau_f$  are frozen variables in the fast time scale). Substituting this into (\ref{eq:deriv_diff_dyn}), the reduced system is 
\begin{align}
\ddot{e}_{nr} + L_p \dot{e}_{nr} + L_i e_{nr} = 0
\label{eq:reduced_system}
\end{align}
together with (\ref{eq:deriv_link})-(\ref{eq:deriv_nom_motor}). Assuming that the reduced system has a unique solution (which is true, for instance, if it is Lipschitz continuous \cite{khalil2002nonlinear}), the Tikhonov's Theorem states that the trajectory of original system approaches to the unique solution as $\epsilon \rightarrow 0$. Note that the solution almost corresponds to the ideal behavior because $e_{nr}=0$ is achieved exponentially in (\ref{eq:reduced_system}). An analysis for the PD-type $C(s)$ can be performed similarly.


\section*{Acknowledgment}
We thank to Julian Klodmann for his great support in experimental setup and for  interesting discussions.
%

\ifCLASSOPTIONcaptionsoff
  \newpage
\fi

\bibliographystyle{IEEEtran}
\bibliography{IEEEabrv,ForArxiv}

\begin{thebibliography}{10}
\providecommand{\url}[1]{#1}
\csname url@rmstyle\endcsname
\providecommand{\newblock}{\relax}
\providecommand{\bibinfo}[2]{#2}
\providecommand\BIBentrySTDinterwordspacing{\spaceskip=0pt\relax}
\providecommand\BIBentryALTinterwordstretchfactor{4}
\providecommand\BIBentryALTinterwordspacing{\spaceskip=\fontdimen2\font plus
\BIBentryALTinterwordstretchfactor\fontdimen3\font minus
  \fontdimen4\font\relax}
\providecommand\BIBforeignlanguage[2]{{%
\expandafter\ifx\csname l@#1\endcsname\relax
\typeout{** WARNING: IEEEtran.bst: No hyphenation pattern has been}%
\typeout{** loaded for the language `#1'. Using the pattern for}%
\typeout{** the default language instead.}%
\else
\language=\csname l@#1\endcsname
\fi
#2}}

\bibitem{verbert2016adaptive}
K.~Verbert, R.~T{\'o}th, and R.~Babu{\v{s}}ka, ``Adaptive friction
  compensation: A globally stable approach,'' \emph{IEEE/ASME Transactions on
  Mechatronics}, vol.~21, no.~1, pp. 351--363, 2016.

\bibitem{olsson1996observer}
H.~Olsson and K.~J. {\AA}str{\"o}m, ``Observer-based friction compensation,''
  in \emph{Proceedings of the 35th IEEE Conference on Decision and Control
  (CDC)}, vol.~4, 1996, pp. 4345--4350.

\bibitem{johanastrom2008revisiting}
K.~Johanastrom and C.~Canudas-De-Wit, ``Revisiting the lugre friction model,''
  \emph{IEEE control Systems}, vol.~28, no.~6, pp. 101--114, 2008.

\bibitem{maged2019dynamic}
M.~Iskandar and S.~Wolf, ``Dynamic friction model with thermal and load
  dependency: modeling, compensation, and external force estimation,'' in
  \emph{IEEE International Conference on Robotics and Automation (ICRA)}, 2019.

\bibitem{kaneko1990motion}
K.~Kaneko, S.~Kondo, and K.~Ohnishi, ``A motion control of flexible joint based
  on velocity estimation,'' in \emph{16th IEEE Annual Conference of Industrial
  Electronics Society}, 1990, pp. 279--284.

\bibitem{ruderman2015observer}
M.~Ruderman and M.~Iwasaki, ``Observer of nonlinear friction dynamics for
  motion control,'' \emph{IEEE Transactions on Industrial Electronics},
  vol.~62, no.~9, pp. 5941--5949, 2015.

\bibitem{spong1990modeling}
M.~W. Spong, ``Modeling and control of elastic joint robots,'' \emph{ASME J.
  Dyn. Syst. Meas. Control}, vol. 109, pp. 310--319, 1990.

\bibitem{zhang1997control}
G.~Zhang and J.~Furusho, ``Control of robot arms using joint torque sensors,''
  in \emph{IEEE International Conference on Robotics and Automation (ICRA)},
  vol.~4, 1997, pp. 3148--3153.

\bibitem{park2007disturbance}
S.-K. Park and S.-H. Lee, ``Disturbance observer based robust control for
  industrial robots with flexible joints,'' in \emph{International Conference
  on Control, Automation and Systems (ICCAS)}, 2007, pp. 584--589.

\bibitem{le2008friction}
L.~Le~Tien, A.~Albu-Sch{\"a}ffer, A.~De~Luca, and G.~Hirzinger, ``Friction
  observer and compensation for control of robots with joint torque
  measurement,'' in \emph{IEEE/RSJ International Conference on Intelligent
  Robots and Systems}, 2008, pp. 3789--3795.

\bibitem{kim2014robust}
M.~J. Kim and W.~K. Chung, ``Robust control of flexible joint robots based on
  motor-side dynamics reshaping using disturbance observer (dob),'' in
  \emph{IEEE/RSJ International Conference on Intelligent Robots and Systems
  (IROS)}, 2014, pp. 2381--2388.

\bibitem{kim2015disturbance}
------, ``Disturbance-observer-based pd control of flexible joint robots for
  asymptotic convergence,'' \emph{Robotics, IEEE Transactions on}, vol.~31,
  no.~6, pp. 1508--1516, 2015.

\bibitem{tomei1991simple}
P.~Tomei, ``A simple pd controller for robots with elastic joints,''
  \emph{Automatic Control, IEEE Transactions on}, vol.~36, no.~10, pp.
  1208--1213, 1991.

\bibitem{khalil2002nonlinear}
H.~K. Khalil and J.~Grizzle, \emph{Nonlinear systems}.\hskip 1em plus 0.5em
  minus 0.4em\relax Prentice hall Upper Saddle River, 2002, vol.~3.

\bibitem{kim2015bringing}
M.~J. Kim, Y.~Choi, and W.~K. Chung, ``Bringing nonlinear $\mathcal{H}_\infty$
  optimality to robot controllers,'' \emph{Robotics, IEEE Transactions on},
  vol.~31, no.~3, pp. 682--698, 2015.

\bibitem{de1995new}
C.~C. De~Wit, H.~Olsson, K.~J. {\AA}str{\"o}m, and P.~Lischinsky, ``A new model
  for control of systems with friction,'' \emph{IEEE Transactions on automatic
  control}, vol.~40, no.~3, pp. 419--425, 1995.

\bibitem{ott2008passivity}
C.~Ott, A.~Albu-Sch{\"a}ffer, A.~Kugi, and G.~Hirzinger, ``On the
  passivity-based impedance control of flexible joint robots,'' \emph{Robotics,
  IEEE Transactions on}, vol.~24, no.~2, pp. 416--429, 2008.

\end{thebibliography}

\end{document}